\documentclass[12pt,reqno]{amsart}

\usepackage[T1]{fontenc}
\usepackage[utf8]{inputenc}

\usepackage{amsmath,amssymb,amsthm}
\usepackage{amsfonts,mathtools}
\usepackage{bm}
\usepackage{physics} 
\usepackage{bbm}
\usepackage{newtxtext,newtxmath} 
\usepackage[margin=1in]{geometry}
\usepackage{graphicx}
\usepackage{subcaption}
\usepackage{booktabs}
\usepackage{multirow}
\usepackage{algorithm,algorithmic}
\usepackage{blkarray}
\usepackage{enumerate}
\usepackage{tikz-cd}
\usepackage{comment}
\usepackage{siunitx}
\usepackage{makecell}

\usepackage[round,authoryear]{natbib}

\usepackage[colorlinks=true,allcolors=blue]{hyperref}

\usepackage[capitalize,nameinlink]{cleveref}

\newcommand{\ind}{\perp\!\!\!\perp}
\newcommand{\K}{\mathbb{K}}
\newcommand{\Z}{\mathbb{Z}}
\newcommand{\R}{\mathbb{R}}

\newtheorem{theorem}{Theorem}[section]
\newtheorem{lemma}[theorem]{Lemma}
\newtheorem{proposition}[theorem]{Proposition}
\newtheorem{corollary}[theorem]{Corollary}

\theoremstyle{definition}
\newtheorem{definition}[theorem]{Definition}
\newtheorem{example}[theorem]{Example}
\newtheorem{remark}[theorem]{Remark}

\numberwithin{equation}{section}

\DeclareMathOperator{\rowspan}{rowspan}

\DeclareMathOperator{\graver}{Gr}

\DeclareMathOperator{\im}{im}
\title{Algebraic Signatures for Structural Learning in Probability Tensors}

\author{Akihiro Maeda}
\address{Graduate School of Arts and Sciences, The University of Tokyo, Tokyo, Japan}
\email{akihiromaeda@e.gcc.u-tokyo.ac.jp}

\author{Shohei Hidaka}
\address{Japan Advanced Institute of Science and Technology, Nomi, Ishikawa, Japan}
\email{shhidaka@jaist.ac.jp}

\author{Satoshi Aoki}
\address{Faculty of Data Science, Shiga University, Hikone, Shiga, Japan}
\email{satoshi-aoki@biwako.shiga-u.ac.jp}


\begin{document}

\begin{abstract}
Algebraic statistics characterizes statistical models through polynomial constraints, but it has mainly been used for analytically specified model classes.
This paper studies the inverse problem: identifying probabilistic structure from vanishing binomials observed in empirical probability tensors.
We treat the vanishing binomials of a toric model as its algebraic signature, and turn the ideal-variety correspondence of algebraic statistics into an operational procedure for structural learning that identifies a model by signature matching without parameter estimation. 
By restricting attention to a computationally tractable class of configuration matrices, which we call {\it the Kronecker-stack class}, we make these signatures explicitly enumerable. 
Within this class we define minimum invariant constraint (MIC) as the atomic unit characterizing each signature and generalizing the notion of independence.
We tested this approach employing MICs on synthetic data as well as on corpus-scale real language data.
The results suggested the utility of the method, revealing that the identified rank-one structures correspond to interpretable sets of words.
These results open up a new avenue for applying algebraic statistics to computational linguistics.
\end{abstract}

\maketitle

\section{Introduction}
\label{sec:introduction}
Algebraic statistics studies statistical models through algebraic varieties and polynomial constraints~\citep{drton2009lectures, aoki2012markov, sullivant2023algebraic}, using tools from algebraic geometry~\citep{cox1997ideals}.
Foundational works such as \cite{geiger2006toric} and \cite{rapallo2007toric} characterized graphical models, independence models and log-linear models through toric ideals and varieties, while \cite{drton2007algebraic} provided a unifying definition of algebraic statistical models (see \cite{kahle2018algebraic} for a concise algebraic overview).
A common approach of algebraic statistics is to proceed from a model to its characterization such as derivation of vanishing ideals or analysis of decomposability, which we describe as a \emph{forward problem}, referring to computation from models to data.
Thus, for statistical application, the operational use of toric characterization has largely centered on Markov bases for fiber sampling and exact testing of contingency tables, though practical implementations are limited due to their computational complexity  ~\citep{hara2012running,alexandr2024decomposable}.

Large-scale data, such as those arising in machine learning fields, are increasingly understood to concentrate on low-dimensional structures embedded in a high-dimensional ambient space, known as \emph{manifold hypothesis} \citep{Belkin06a, meilua2024manifold, whiteley2026statistical}.
Those observed phenomena are inherently geometrical, though insights from algebraic geometry are not yet fully exploited in their studies~\citep{marchetti2025position}.
To elucidate an underlying structure for such data, one must solve an \emph{inverse problem} in which one needs to infer some algebraic structure behind an observed geometric distribution, i.e., identifying an algebraic model that generated the data with inverted direction from data to a model with vanishing ideals.
In graphical modeling, structure learning methods have been developed to identify a model from data, 
typically by selecting a model through maximization of likelihood-based optimization (score-based search) or  conditional independence testing (constraint-based method)~\citep{spirtes2000causation,drton2017structure}. 
However, both methods are forward approaches in our view  since they posit one known candidate structure and test whether the data are consistent with it, thus incapable of identifying an unknown structure.
Also, tools of algebraic geometry such as vanishing binomials are not fully employed in existing structural learning.

Our approach is to revisit the ideal-variety bijective correspondence of algebraic geometry as the basis for identifying statistical model structures via their algebraic constraints.
Specifically, we propose to exploit vanishing binomials  as an algebraic signature of a generating model.
Since a toric model is uniquely defined by the  vanishing pattern of its ideals, 
we reformulate the structural learning problem as constraint-based signature matching without parameter estimation.
To make this operational, 
we employ a linear-algebraic representation of a toric model, and classify candidate structures by row-space equivalence of configuration matrices, placing various model families such as independence, hierarchical and context-specific models in one common algebraic space.
The advantages of our linear-algebraic representation arise from two properties: 
parameter invariance, which permits direct structural identification by reading off vanishing binomials as algebraic signatures, thereby circumventing the need for parameter optimization; and 
expressivity,  which captures higher-order interactions beyond the pairwise relations imposed by graphical models.

We introduce a computationally tractable theoretical framework built on two key results: 
a bijective correspondence between a class of toric models and their algebraic signatures consisting of Graver basis in Theorem~\ref{thm:bijectivity}, and 
an explicit enumeration of these signatures within a defined model class, the Kronecker-stack class in Theorem~\ref{thm:factorwise_integer_kernel_generation}.
We propose Algebraic Signature Analysis (ASA), a structural learning procedure operating in two modes: 
ASA-M for confirmatory model selection via signature matching, and 
ASA-D for exploratory detection of atomic structural units.
We validate ASA through two experiments: 
a proof-of-concept on synthetic discrete data demonstrating model selection purely from empirical vanishing patterns, and 
an application to large-scale linguistic data uncovering interpretable semantic structures, opening new application avenues for algebraic statistics.

Vanishing ideals have already been utilized in the machine learning community for representation learning: 
\cite{livni2013vanishing} proposed \emph{Vanishing Component Analysis (VCA)} as an algorithm to extract approximate polynomial constraints from continuous data, followed by \cite{Kera2018ApproximateVIA, Wirth2022} and others who proposed enhanced algorithms, 
whereas \cite{pelleriti2025approximating} integrated vanishing ideal algorithms into neural networks to approximate latent manifolds.
Although these prior studies assume that data classes correspond to disjoint algebraic varieties, they fundamentally utilize vanishing polynomials merely as a non-linear feature extraction step to make the data linearly separable, rather than aiming to recover the true geometric structure of the generative models. 
In contrast, our vanishing binomials arise combinatorially from 
the geometric structure of the generative model, determined by the duality between the row space and kernel of the configuration matrix.

In algebraic statistics, vanishing ideals as polynomial constraints have been actively studied for the characterization of specific model classes, including:
staged trees~\citep{DUARTE2020127},
intervention models of directed acyclic graphs~\citep{duarte2025algebraic},
CStrees~\citep{alexandr2024decomposable,duarte2026representation}, and
Gaussian graphical models~\citep{misra2021gaussian}.
For deep learning, \cite{kubjas2024geometry} derived the vanishing ideals of neural network functional spaces via Zariski closure (\emph{neurovarieties}). 
Notably, \cite{boege2024realbirationalimplicitizationstatistical} generalized statistical models parameterized by birational maps, 
derived their vanishing ideals via birational implicitization, 
establishing a unified foundation for constraint-based structure learning.
However, when translating these algebraic characterizations into practical applications, they remain confined to the realm of existing constraint-based methods. 
These methods still fundamentally rely on a forward approach of hypothesis testing—verifying empirical data against pre-derived catalogs ``to choose from a finite collection of graphs the one whose model fits a given sample best''~\citep{boege2024realbirationalimplicitizationstatistical}. 

In contrast, our proposed approach shifts this paradigm. 
By representing the ambient spaces of models via equivalence classes quotiented by the row spaces of configuration matrices, our framework becomes uniquely capable of 
(i) delineating the ambient space of all possible geometric structures 
prior to hypothesis testing, and 
(ii) rigorously determining algebraic and geometric relations among models, such as hierarchical inclusions.

The remainder of this paper is organized as follows. 
Section 2 establishes the foundational algebraic approach for model selection, defining the probability distributions, statistical models, and their algebraic signatures. 
Section 3 introduces the Kronecker-stack class of configuration matrices and describes the factor-wise generation of their integer kernels along with minimum invariant constraints (MICs). 
Section 4 details the proof-of-concept (PoC) numerical experiments for structural model selection through algebraic signature matching and threshold calibration. Section 5 extends the framework to an empirical domain, demonstrating a practical linguistic application to a corpus-scale natural language dataset. 
Finally, Section 6 provides a discussion on the structural implications and limitations of the proposed approach and Section 7 concludes the paper.

\section{Algebraic approach for model selection}
\label{sec:primer}

\subsection{Probability distributions}
We denote a probability simplex with $N$ states by 
\begin{equation}
    \Delta_{N-1}=\Big\{\bm{p}=(p_1,\ldots, p_N)\in \R_{\geq 0}^N\mid \sum_{i=1}^Np_i=1\Big\}.
\end{equation}
We also define the interior of a probability simplex by $\Delta_{N-1,>0}:=\Delta_{N-1}\cap \R_{>0}^{N}$ where all entries of the probability vector are positive.
We consider a family of probability distributions $\bm{p}=(p_1,\ldots, p_N)^\top \in \Delta_{N-1}$ of the form:
\begin{equation}\label{eq:log-linear}
    \log{p_i(\bm{\theta})} = \sum_{j=1}^d\theta_ja_j(i)-\psi(\bm{\theta}),\qquad i = 1,\ldots, N
\end{equation}
where  $\bm{\theta}=(\theta_1,\ldots, \theta_d)$ is a parameter vector, $\psi(\bm{\theta})=\log{\sum_{i=1}^N\exp{\sum_{j=1}^d\theta_ja_j(i)}}$ is the normalizing constant.
This family is often called a \emph{log-linear model} that is a subset of the exponential family~\citep{sullivant2023algebraic}.
By setting the column vectors $\bm{a}(i)=(a_1(i),\ldots, a_d(i))^{\top}\in \Z^d$ for $i=1,\ldots, N$ and arranging them horizontally into a $d\times N$ integer matrix $A$ as
\begin{equation}
    A=\begin{bmatrix}
        \bm{a}(1),  \bm{a}(2),\ldots,  \bm{a}(N)
    \end{bmatrix},
\end{equation}
we can rewrite (\ref{eq:log-linear}) in the row vector form
\begin{equation}\label{eq:toric}
    \bm{p}(\bm{\theta})^{\top}
    =
    \frac{1}{Z}\exp(\bm{\theta}^{\top}A),
\end{equation}
where the exponential is applied elementwise and
$
Z=\sum_{i=1}^N \exp(\bm{\theta}^{\top}\bm{a}(i)).
$
We call this matrix $A$ a \emph{configuration matrix}.
We say that the family of distributions $\{p_i(\bm{\theta})\}$ defined by 
\eqref{eq:log-linear} is \emph{homogeneous} if the all-one vector 
$\bm{1}_N:=(1,\ldots,1)^\top\in\R^N$ lies in the row space of $A$; 
in this case, we call it a \emph{toric model}.
Later in this study, we restrict ourselves to the binary case $A\in\{0,1\}^{d\times N}$, viewing $A$ as an incidence matrix encoding whether each factor acts on each state.

\begin{definition}
Let $A\in \{0,1\}^{d\times N}$ be the configuration matrix.
The set of integer vectors in the kernel of $A$ is defined as
\begin{equation}
    \ker_{\Z}A:= \{z \mid Az=0,\; z\in \Z^N\}.
\end{equation}
We call $\ker_{\Z}A$ the \emph{integer kernel} of $A$.
\end{definition}

The integer kernel $\ker_{\Z}A\subseteq \Z^N$ is closed under integer multiplication and addition: for $z_1, z_2 \in \ker_{\Z}A$
\begin{equation}
    az_1 + bz_2 \in \ker_ZA,\quad\forall a,b \in \Z
\end{equation}
Due to this property, $\ker_{\Z}A$ is also called  an \emph{integer lattice}~\citep{aoki2012markov}. 

Also, we define a row space of the matrix $A$ as $\rowspan(A) :=\{u \in \R^N\mid u=\sum_{i=1}^d c_i A(i),\; c_i\in \R\}$ that is the vector space spanned by the rows $A(i),\;(i=1,\ldots,d)$ of the matrix $A$.
We remind the basic fact of linear algebra that we can identify the kernel of $A$ with its row space:
\begin{equation}
    \rowspan(A)^{\bot} = \ker A; \quad \rowspan(A) = (\ker A)^{\bot},
\end{equation}
where $\bot$ denotes the orthogonal complement~\citep{aoki2012markov}. 
By this duality, we equate the row space with the kernel space.

\subsection{Models and signatures}
We consider a set of probability distributions associated with a configuration matrix $A$ that shares its common row space as a model.

\begin{definition}[Model; Definition 1.1.9 in \cite{drton2009lectures} ]\label{def:model}
    Given a matrix $A\in \{0,1\}^{d\times N}$, 
    a \emph{model} of $A$ is defined by 
    \begin{equation}
        M(A):=\Big\{\bm{p}\in \Delta_{N-1,>0}\mid\log \bm{p} \in \rowspan(A)\Big\}.
    \end{equation}
    
\end{definition}
Note that a model depends only on the row space of $A$:
if $Q \in \mathbb{R}^{d \times d}$ is invertible, then $\rowspan(QA) = \rowspan(A)$,
so $A$ and $QA$ define the same model.
Geometrically, the row space is invariant under coordinate transformations in the parameter space $\mathbb{R}^d$.

\begin{definition}[Graver basis]
Let $\sqsubseteq$ be a partial order defined on $\Z^n$ as: For $\bm{u,v} \in \Z^n$
\begin{equation}
    \bm{u} \sqsubseteq \bm{v} \overset{\mathrm{def}}{\iff} u_iv_i\geq 0\; \mathrm{and}\; |u_i| \leq |v_i|\,\quad \forall i=1,\ldots, n.
\end{equation}
A \emph{conformal decomposition} of $\bm u$ is an expression $\bm{u}=\bm{v}+\bm{w}$, 
where $\bm{v,w} \in \Z^N\setminus {\bm{0}}$ and both of $\bm v \sqsubseteq \bm u$  and  $\bm w \sqsubseteq \bm u$ hold.
A \emph{Graver basis} of an integer lattice is the set of the minimal elements with respect to the conformity partial order $\sqsubseteq$.
No element of a Graver basis has a conformal decomposition.
\end{definition}       

Let $p_1, \ldots, p_N$ be indeterminates.
A  \emph{monomial} in these indeterminates is an expression of the form
\begin{equation}
    \bm{p}^{\bm{u}}:=p_1^{u_1}p_2^{u_2}\cdots p_N^{u_N}
\end{equation} where $\bm{u}=(u_1, \ldots, u_N)\in \mathbb{Z}_{\geq 0}^N$ is a vector of nonnegative integers.

\begin{definition}[Signature]\label{def:signature}
Given a matrix $A\in \Z^{d\times N}$ of a toric model, we define a signature for its model $M(A)$ by a set of binomials determined by $A$ as below:
\begin{equation}
    \Sigma(A):=\Big\{\bm{p}^{\bm{u}+}-\bm{p}^{\bm{u}-}\mid \bm{u}\in \graver(A);\; \bm{u} = \bm{u}^+-\bm{u}^-;\; \bm{u}^+,\bm{u}^-\in {\Z_{\geq 0}^N} \Big\},
\end{equation}
where $\graver(A)$ is a Graver basis of the integer lattice  $\ker_{\Z}A$ associated with $A$ and $\bm{u}^+, \bm{u}^-$ have disjoint supports.
We denote the signature of a configuration matrix $A$ as $\Sigma(A)$.
\end{definition}

\begin{theorem}[Bijectivity of models and signatures]
\label{thm:bijectivity}
For homogeneous configuration matrices $A$ and $A'$,
the following bijective correspondence among their models $M(A),\;M(A')$ and their signatures $\Sigma(A),\;\Sigma(A')$ holds.
\begin{equation}
    M(A) = M(A') \iff \Sigma(A)=\Sigma(A')
\end{equation}
\end{theorem}

\begin{proof}
    For $(\Rightarrow)$, assume $M(A) = M(A')$. 
    By Definition~\ref{def:model} and the homogeneity assumptions $\mathbf{1}\in\operatorname{rowspan}(A)$ and $\mathbf{1}\in\operatorname{rowspan}(A')$, the equality $M(A)=M(A')$ entails $\operatorname{rowspan}(A)=\operatorname{rowspan}(A')$ by the supplemental Lemma~\ref{lem:model-rowspace-equivalence} in Appendix~\ref{app:proof2.5}.
    Since the kernel of a matrix is the orthogonal complement of its row space, we have $\ker_{\R}A = \ker_{\R}A'$, which restricts to the integer lattices as $\ker_{\Z}A = \ker_{\Z}A'$. 
    Since their integer kernels are identical, their Graver bases must coincide, i.e., $\graver(A) = \graver(A')$. 
    Because the signature $\Sigma(A)$ is uniquely determined by the Graver basis, we conclude $\Sigma(A) = \Sigma(A')$.

    For $(\Leftarrow)$, assume $\Sigma(A) = \Sigma(A')$. 
    By Definition~\ref{def:signature}, each binomial in the signature is of the form $\bm{p}^{\bm{u}^{+}} - \bm{p}^{\bm{u}^{-}}$, where $\bm{u}\in \graver(A)$ and $\bm{u}=\bm{u}^{+}-\bm{u}^{-}$ with disjoint supports.
    Hence, we can uniquely recover their exponent difference vectors up to sign convention as $\bm{u} = \bm{u}^{+} - \bm{u}^{-}$  that collectively form the Graver basis.
    Since the Graver basis is a conformal generating set of the lattice kernel  $\ker_Z A$; that is, 
\begin{equation}
    \langle\graver(A)\rangle_{\Z}=\ker_{\Z}A.
\end{equation}
where $\langle\graver(A)\rangle_{\Z}$ denotes the set of all finite integer linear combinations of elements of $\graver(A)$.
See a supplemental Lemma~\ref{lem:graver-generates-lattice} in Appendix~\ref{app:proof2.5}.

Since $A$ and $A'$ have integer entries, their real kernels are spanned by rational kernel vectors, and clearing denominators turns such vectors into integer kernel vectors. 
Thus each real kernel is the real linear span of its corresponding integer kernel. See a supplemental Lemma~\ref{lem:integer-kernel-spans-real}. 
Since $\ker_{\Z}A=\ker_{\Z}A'$, it follows that $\ker_{\R}A=\ker_{\R}A'$. 
Taking the orthogonal complements, we obtain $\rowspan(A)  = \rowspan(A')$. 
Therefore, the conditions $\log \bm{p} \in \rowspan(A)$ and $\log \bm{p} \in \rowspan(A')$ are entirely equivalent, which, by Definition~\ref{def:model}, concludes $M(A) = M(A')$.
\end{proof}

This argument reverses the standard toric variety-ideal correspondence described by ~\citet[Chapter~4]{sturmfels1996grobner}: 
for an integer configuration matrix $A$, the toric ideal $I_A$ is described by binomials 
$\bm{p}^{\bm{u}^{+}}-\bm{p}^{\bm{u}^{-}}$ associated with lattice relations 
$\bm{u}\in\ker_{\mathbb{Z}}A$, and Graver elements correspond to primitive binomials.
Hence, rather than being merely a subset of $I_A$, the Graver-binomial signature $\Sigma(A)$ fully determines the homogeneous toric model up to row-space equivalence.

This bijectivity theorem established above bridges abstract algebraic geometry and statistical inference: it guarantees that a theoretical model can be uniquely identified purely by evaluating its corresponding signature as invariant constraints against empirical data. 
In the context of machine learning, a fundamental challenge is \emph{model selection} (or model identification)---the task of discovering the true underlying structure from a vast hypothesis space of candidate models. 
To mathematically formalize this statistical motivation, we aggregate these testable constraints into a single overarching object. 
By introducing the following definition of a ``Dictionary,'' we complete the foundational theoretical framework necessary for algorithmic model identification.

\begin{definition}[Dictionary]
    Let $\mathcal{C} = \{M(A_k)\}_{k=1}^K$ be a class of candidate toric models parameterized by configuration matrices $A_k \in \mathbb{Z}^{d \times N}$. 
    We define a \emph{dictionary} $\mathcal{D}$ for identifying a model from the class $\mathcal{C}$ as the union of the signatures of all models in $\mathcal{C}$:
    \begin{equation}
        \mathcal{D} := \bigcup_{M(A_k) \in \mathcal{C}} \Sigma(A_k).
    \end{equation}
\end{definition}

\begin{corollary}[Dictionary identifiability]
\label{cor:dictionaly-identifiability}
Let $\mathcal{D} = \bigcup_{k = 1}^K \Sigma(A_k)$ be the dictionary of a finite candidate class $\{A_1,...,A_K\}$. 
Then the binary membership vector induced by $\mathcal{D}$ uniquely identifies each model $M(A_k)$ in the candidate class.
\end{corollary}

\begin{proof}
By bijective correspondence between model and signature from Theorem~\ref{thm:bijectivity}.
\end{proof}

\section{The Kronecker-Stack Class}
\label{sec:main}
\subsection{Definition and scope}
\label{sec:stack-definition}
In this section, we restrict attention to a class of models whose configuration matrices admit a particularly tractable form: a vertically stacked block structure expressible as a Kronecker product of three types of factors. 
Within this class, the signatures can be enumerated explicitly, enabling systematic model selection.

\begin{definition}[Kronecker-stack class]
\label{def:kronekcer-class}
    Let 
    $E_{n_k}$ be a $n_k\times n_k$ identity matrix, $\bm{1}_{n_k}=(1,1,\ldots, 1)^{\top}$ be an all-one vector in the $n_k$ dimensional vector space and
    $e_{i(n_k)}$ be a standard basis vector with $i$-th element equals to $1$ and $0$ otherwise.
    We consider a class of matrices that can be expressed as a vertical stack of row block matrices $A^{(1)},A^{(2)},\ldots, A^{(R)}$ with the same column size $N$ and possibly different row size $m_r\;(r=1,\ldots, R)$:
    \begin{equation}
        A = \begin{bmatrix}
            A^{(1)}\\\vdots\\ A^{(R)}
        \end{bmatrix},
    \end{equation}
    where each matrix $A^{(r)}$ for $1\leq r \leq R$ is further decomposable as the Kronecker product $\otimes$ of finite factors $B_k^{(r)}\;(k=1,\ldots, m)$ for a fixed $m$ over $r=1,\ldots, R$ 
    \begin{equation}
        A^{(r)}= \bigotimes_{k=1}^m B_k^{(r)},
    \end{equation}
    and each factor $B_k^{(r)}$ is of the form:
    \begin{equation}\label{eq:choices-factors}
        B_k^{(r)}\in \big\{E_{n_k},\bm{1}_{n_k}^{\top}, e_{i(n_k)}^{\top}\big\}.
    \end{equation}
Note that  $B_k^{(r)}$ has the same column size $n_k$ across different $r$ for $r=1,\ldots, R$ since any choice in Eq.~\eqref{eq:choices-factors} has the same column size and $\prod_{k=1}^mn_k=N$. 
We call $A^{(r)}$ the $r$-th row block, and $B_k^{(r)}$ the factor in mode $k$ (i.e. the $k$-th factor position) of block $r$.

We only consider a class of models of probability distributions given by configuration matrices defined by Definition \ref{def:kronekcer-class}, named as \emph{Kronecker-stack class} denoted $\mathcal{C}$.
Though we limit the scope of consideration, this class is sufficiently large so that it contains most of probability distributions discussed in the literature such as algebraic statistics and graphical models as we show soon.
This Kronecker-stack class represents multivariate discrete distribution $X=(X_1,\ldots, X_m)$.
Each block matrix $A^{(r)}$ is decomposed as a Kronecker product of $m$ factors that correspond to $m$ random variables $X_k\;(k=1,\ldots, m)$.
Each random variable $X_k$ has $n_k$ states.
Each block $A^{(r)}$ encodes a specific interaction among variables in the probability distribution.

An identity $E_{n_k}$ gives rise to \emph{free} elements, since the $k$-th 
variable does not enforce any restriction as a factor. 
The all-one vector 
$\bm{1}_{n_k}^{\top}$ gives rise to \emph{copy} elements, since it forces 
the corresponding states to be subject to the same effect;
we accordingly 
refer to $\bm{1}_{n_k}$ as a \emph{copy vector} in this construction. 
Finally, $e_{i(n_k)}^{\top}$ gives rise to \emph{selector} elements, 
indicating a selection or filter for a specific state of an individual 
variable, used for localization.
We require that every occurrence of selector factors $e_{i(n_k)}^{\top}$ 
is \emph{selector-complete}: if $e_{i(n_k)}^{\top}$ appears as the $k$-th 
factor of some row block, then every $e_{j(n_k)}^{\top}$, $j=1,\ldots,n_k$, 
must appear as the $k$-th factor of \emph{some} row block (possibly 
different row blocks, and with the other factors free to vary).
\end{definition}
    
The advantage of limiting to this class is to trace the structure of the kernel by factor by factor.
This allows us to exhaustively enumerate a generating set of the kernel lattice for each configuration matrix and identify a signature for each model.
By preparing a dictionary $\mathcal{D}:=\bigcup_{M(A_k)\in \mathcal{C}}\Sigma(A_k)$, we can test all the possible signatures against the data.

\vspace{0.5cm}
The Kronecker-stack class covers a sufficiently broad and practically relevant subclass of finite discrete toric models, including hierarchical log-linear models and elementary context-specific constructions.

\begin{definition}[Hierarchical model~\citep{aoki2012markov}]
Let $[m]=\{1,\ldots,m\}$ index a collection of $m$ random variables, where the $k$-th variable takes values in $[n_k]:=\{1,\ldots,n_k\}$, and $D_1,\ldots, D_r$ be subsets of $[m]$ such that there is no inclusion relation between $D_i$ and $D_j$ for $i\neq j$ with $\bigcup_{i=1}^rD_i=[m]$.
We denote $\mathcal{D}=\{D_1,\ldots, D_R\}$. 
A \emph{hierarchical model} with the generating class $\mathcal{D}$ is defined by a probability vector $\bm{p}\in\Delta_{N-1,>0}$ (where $N=\prod_{k=1}^m n_k$)  that satisfies
\begin{equation}
    \log \bm{p}(\bm{i})=\sum_{D\in \mathcal{D}}\bm\mu_D(\bm{i}_D),
\end{equation}
where $\bm{i}\in \prod_{k=1}^m[n_k]$ is a multi-index for a cell of $m$-way joint probability tensor, $\bm{i}_D:=(i_k)_{k\in D}$ denotes the sub-vector of $\bm i$ obtained by 
restricting to the coordinates in $D\subset[m]$.
$\bm\mu_D$ is a function depending only on the marginal cell $\bm{i}_D$.
\end{definition}

Note that $\Gamma:=\{D\subseteq [m]\mid D\subseteq D_i\mathrm{\;for\,some\,} D_i\in \mathcal{D}\}$ is a \emph{simplicial complex} whose elements are called the \emph{faces}.
Since no $D_i$ is contained in another (by assumption), $\mathcal{D}$ coincides 
precisely with the set of inclusion-maximal faces of $\Gamma$, called \emph{facets}. 
Conversely, since $[m]$ is finite, every face of $\Gamma$ is contained in some facet, so $\Gamma$ is recovered as the downward closure of $\mathcal{D}$; 
in particular, a simplicial complex---and hence a hierarchical model---is fully determined by its facets~\citep{sullivant2023algebraic}.

\begin{proposition}[Expressivity of Kronecker-stack class]
\label{prop:expressivity}
Every hierarchical model of a simplicial complex with defining facets $\mathcal{D}$ belongs to the Kronecker-stack model class $\mathcal{C}$.

Specific models under the hierarchical model in case of 3 variables $X=(X_1,X_2,X_3)$ are listed in Table~\ref{tab:hierarchical_models} and all of them belong to the class $\mathcal{C}$.

\begin{table}[h]
    \centering
\begin{tabular}{c|c|c}
    \hline
        \textbf{Model} & \textbf{Independence} & \textbf{Defining Facets $\mathcal{D}$} \\
        \hline
         Saturated  &  (No independence) & $\{\{1, 2, 3\}\}$ \\
          No3Way  & (No three-way interaction) & $\{\{1, 2\}, \{1, 3\}, \{2, 3\}\}$ \\
         Conditional independence & $X_2 \ind X_3 \mid X_1$ & $\{\{1,2\}, \{1,3\}\}$ \\
        Conditional independence  & $X_1 \ind X_3 \mid X_2$ & $\{\{1, 2\}, \{2, 3\}\}$ \\
        Conditional independence & $X_1 \ind X_2 \mid X_3$ & $\{\{1, 3\}, \{2, 3\}\}$ \\
        Joint independence& $(X_1, X_2) \ind X_3$ & $\{\{1, 2\}, \{3\}\}$ \\
        Joint independence  & $(X_1, X_3) \ind X_2$ & $\{\{1, 3\}, \{2\}\}$ \\
        Joint independence & $(X_2, X_3) \ind X_1$ & $\{\{2, 3\}, \{1\}\}$ \\
         Complete independence  & $X_1 \ind X_2 \ind X_3$ & $\{\{1\}, \{2\}, \{3\}\}$ \\
        \hline
    \end{tabular}
    \caption{All hierarchical models for three-variable distributions}
\label{tab:hierarchical_models}
\end{table}
\end{proposition}

\begin{proof}
Any hierarchical model is specified by its defining facets $\mathcal{D}=\{D_1,\ldots,D_R\}$.
For each facet $D_r$, $r=1,\ldots,R$, construct the block matrix
\begin{equation}
    A^{(r)}=\bigotimes_{k=1}^{m}B_k^{(r)}, \qquad
    B_k^{(r)}=\begin{cases} E_{n_k} & k\in D_r,\\ \bm{1}_{n_k}^\top & k\notin D_r,\end{cases}
\end{equation}
and let $A$ be the vertical stack of $A^{(1)},\ldots,A^{(R)}$. 
Since every factor is an identity matrix or an all-one row vector, $A$ belongs to the Kronecker-stack class $\mathcal{C}$.

Let $\bm p$ be any distribution in the hierarchical model with generating class $\mathcal{D}$, 
so that $\log \bm p(\bm i)=\sum_{D\in\mathcal{D}}\bm\mu_D(\bm i_D)$ for some functions $\bm\mu_D$. 
We show that $\bm p$ is realized by $A$.
Note that the columns of $A^{(r)}$ are indexed by the cell $\bm i\in\prod_{k=1}^m[n_k]$, and its rows by $\bm j\in\prod_{k\in D_r}[n_k]$, 
since the factors outside $D_r$ are the single-row vectors $\bm{1}_{n_k}^\top$. 
By the entrywise rule for Kronecker products, the entry of $A^{(r)}$ in row $\bm j$ and column $\bm i$ is
$A^{(r)}_{\bm j,\bm i} = \mathbb{I}[\bm i_{D_r}=\bm j]$,
where $\mathbb{I}$ is an indicator function. 
Equivalently, each row $\bm j$ of $A^{(r)}$ picks out the fiber over $\bm j$, that is, all cells $\bm i$ agreeing with $\bm j$ on the coordinates in $D_r$.

Define the parameter vector $\bm\theta^{(r)}:=\big(\bm\mu_{D_r}(\bm j)\big)_{\bm j}$ using the functions $\bm\mu_{D_r}$ given above, and 
let $\bm\theta=(\bm\theta^{(1)},\ldots,\bm\theta^{(R)})$ be the stacked parameter vector.
By the fiber structure, only the term $\bm j=\bm i_{D_r}$ survives in each block sum, so
\begin{equation}
\big(A^{\top}\bm\theta\big)_{\bm i}
=\sum_{r=1}^{R}\big((A^{(r)})^{\top}\bm\theta^{(r)}\big)_{\bm i}
=\sum_{r=1}^{R}\bm\mu_{D_r}(\bm i_{D_r})
=\log \bm p(\bm i).
\end{equation}
Thus $\bm p$ is realized by the configuration matrix $A\in\mathcal{C}$ with parameter
$\bm\theta$. 
Since $\bm p$ was an arbitrary distribution in the hierarchical model, every hierarchical model belongs to $\mathcal{C}$.
\end{proof}

In the simplest case of 2 binary variables where $m=2$ and $n_k=2\;(k=1,2) $, the possible facets are either $\mathcal{D}_1=\{\{1,2\}\}$ where $R=1$ or $\mathcal{D}_2=\{\{1\},\{2\}\}$ with $R=2$. 
Following the construction shown in the above proof yields a configuration matrix $A_1$ for the facets $\mathcal{D}_1$ in Eq.~\eqref{eq:config_saturated} and $A_2$ for  $\mathcal{D}_2$ in Eq.~\eqref{eq:config_independnt}:
\begin{align}
\label{eq:config_saturated}
    A_1 &= \begin{bmatrix}A^{(1)}\end{bmatrix}=\begin{bmatrix}
        E_{2}\otimes E_{2}
    \end{bmatrix} = \begin{bmatrix}
        1&0&0&0\\
        0&1&0&0\\
        0&0&1&0\\
        0&0&0&1\\
    \end{bmatrix}\\
    \label{eq:config_independnt}
    A_2 &= \begin{bmatrix}A^{(1)}\\A^{(2)}\end{bmatrix}=\begin{bmatrix}\begin{bmatrix}
        E_{2}\otimes \bm{1}_{2}
    \end{bmatrix}\\
    \begin{bmatrix}
        \bm{1}_{2}\otimes E_{2}
    \end{bmatrix}\end{bmatrix} = \begin{bmatrix}
        1&1&0&0\\
        0&0&1&1\\
        1&0&1&0\\
        0&1&0&1\\
    \end{bmatrix}.
\end{align}

\begin{remark}[Scope of the Kronecker-stack class]
\label{rem:expressivity_scope}
Hierarchical models subsume all decomposable graphical models, including independence models~\cite[Chapter~4.4]{lauritzen1996}.
Hierarchical models formalize any interactions at variable level that can be expressed by the configuration matrix in the Kronecker-stack class as we just show in Proposition~\ref{prop:expressivity}.
The Kronecker-stack class also covers elementary context-specific constructions in which a selector $e_i^\top$ localizes the model to a specific state or stratum beyond at the variable level.
Context-specific independence models~\citep{boutilier1996context} such as $X_2\ind X_3\mid X_1=1$ in the $2\times 2\times 2$ variable setting is expressed as Kronecker-stack matrix in Eq.~\eqref{eq:CSI}:
\begin{equation}
    \label{eq:CSI}A_{CSI}=\begin{bmatrix}A^{(X_1=1)}\\A^{(X_1=2)}\end{bmatrix}=
        \begin{bmatrix}\begin{bmatrix}
            e_{1(2)}^\top\otimes E_{2}\otimes \bm{1}_2^\top\\
            e_{1(2)}^\top\otimes \bm{1}_{2}^\top\otimes E_2\\\end{bmatrix}\\\begin{bmatrix}
            e_{2(2)}^\top\otimes E_{2}\otimes E_2\\\end{bmatrix}
        \end{bmatrix}=
        \begin{bmatrix}\begin{bmatrix}
            1&1&0&0& 0&0&0&0\\
            0&0&1&1& 0&0&0&0\\
            1&0&1&0& 0&0&0&0\\
            0&1&0&1& 0&0&0&0\\\end{bmatrix}\\\begin{bmatrix}
            0&0&0&0& 1&0&0&0\\
            0&0&0&0& 0&1&0&0\\
            0&0&0&0& 0&0&1&0\\
            0&0&0&0& 0&0&0&1\\\end{bmatrix}\end{bmatrix}.
\end{equation}

Some other context-specific variations also fall into the Kronecker-stack class.
In particular, split models~\citep{hojsgaard2003split} can be represented in this class because their context-specific conditions are specified stratum by stratum along a tree structure: each branching is selected by a standard basis vector, and the blocks are simply stacked without interaction.

On the other hand, more relaxed and complicated models such as the stratified graphical model~\citep{nyman2014stratified} appear to fall outside the Kronecker-stack class. 
For instance, for three binary variables with context-specific independences 
$X_2\ind X_3\mid X_1=0$ and $X_1\ind X_3\mid X_2=0$, 
one of the row blocks of the configuration matrix takes the form
\begin{equation}
    \bm e_1^\top\otimes\bm e_1^\top\otimes E_2
    \;+\;
    \bm e_1^\top\otimes\bm e_2^\top\otimes E_2
    \;+\;
    \bm e_2^\top\otimes\bm e_1^\top\otimes E_2,
\end{equation}
arising from parameter tying between the two context-specific conditions above. 
Unlike a split model, where each condition carves out its own branch along a common tree, here 
the two conditions carve out overlapping, cross-cutting combinations of states across $X_1$ and $X_2$: 
the block above merges three such combinations into a single tied row. 
Each summand is individually a Kronecker product, but their sum is not, since no single 
choice of $B_1,B_2,B_3$ reproduces this entangled combination as $B_1\otimes B_2\otimes 
B_3$. 
A precise characterization of such mixed structures, and the computation of their kernels, is left as future work.
\end{remark}

\subsection{Kernel of Kronecker-stack class}
The aim here is to explicitly derive a Graver basis of a given matrix under the Kronecker-stack class for the enumeration of signatures.
First, we list some useful lemmas to derive them.

\begin{lemma}[Kernel of stack matrices]
\label{lem:kernel-stack}
Let $A\in \K^{m_1\times n}, B \in \K^{m_2\times n}$ be matrices where $\K$ is a field. 
Then, the kernel of the stacked matrix of those two matrices is the intersection of the two kernels of each matrix:
\begin{equation}
    \ker\left(\begin{bmatrix}
        A\\B
    \end{bmatrix}\right) = \ker A \cap \ker B
\end{equation}
\end{lemma}
\begin{proof}
Let $\bm{u}\in \K^{n}$ be a vector and  $\bm{0}_k$ be the zero vector in the $k$-dimensional space.
\begin{equation}
     \begin{bmatrix}
            A\\B
        \end{bmatrix}\bm{u}=\bm{0}_{m_1+m_2}\iff A\bm{u}=\bm{0}_{m_1}\;\mathrm{and}\;B\bm{u}=\bm{0}_{m_2} \iff \bm{u} \in \ker A \cap \ker B.
\end{equation}
\end{proof}

\begin{definition}[Direct sum of subsets]
\label{def:direct-sum}
    Let $V$ be a vector space and $W_1,W_2$ be subspaces of $V$, then we denote : 
\begin{equation}
    W_1+W_2=\{w_1+w_2\mid w_1\in W_1,\; w_2\in W_2\}
\end{equation}
If, in addition, $W_1\cap W_2=\{\bm{0}\}$, then $W_1+W_2$ is called a \emph{direct sum}, written $W_1\oplus W_2$.
\end{definition}
$W_1+W_2$ is the smallest subspace of $V$ containing both of $W_1$ and $W_2$.

\begin{lemma}[Integer kernel of Kronecker products]
\label{lem:kronecker}
Let $A\in \Z^{m_1\times n_1},\; B \in \Z^{m_2\times n_2}$ be integer matrices.
Then, 
\begin{equation}
    \ker_{\Z}(A\otimes B)=\ker_{Z} {A}\otimes \Z^{n_2}+\Z^{n_1}\otimes \ker_{Z}{B}
\end{equation}
\end{lemma}
\begin{proof}
The inclusion, 
$\ker A\otimes \Z^{n_2}+\Z^{n_1}\otimes \ker B\subseteq\ker(A\otimes B)$, 
is immediate. 
If we take $\bm{u}\in \ker_{\Z} A, \bm{v} \in \ker_{\Z} B, \bm x \in \Z^{n_1}, \bm y\in \Z^{n_2}$,
\begin{equation}
(A\otimes B)(\bm{u}\otimes \bm{y}+\bm{x}\otimes \bm{v})=A\bm{u}\otimes B\bm{y}+A\bm{x}\otimes B\bm{v}= \bm{0}\otimes B\bm{y} +A\bm{x} \otimes \bm{0}=\bm{0}.
\end{equation}

Conversely, 
since $\ker_{\Z}A=\ker A\cap \Z^{n_1}$ is saturated sublattice of $\Z^{n_1}$ and, $\Z^{n_1} / \ker_{\Z}A \cong \im_{\Z}(A)$, $\im_{\Z}(A)$ is a subgroup of the free abelian group $\Z^{m_1}$, hence torsion-free.
Therefore, by \citet[Exercise 1.3.5]{CoxLittleSchenck2011}, 
there exists $U\subseteq \Z^{n_1}$ such that $\Z^{n_1}=\ker_{\Z}A\oplus U$. 
Similarly, we take $V\subseteq \Z^{n_2}$ such that $\Z^{n_2}=\ker_{\Z}B\oplus V$.
Then,
\begin{equation}
\Z^{n_1}\otimes \Z^{n_2}
=
(\ker_{\Z} A\otimes \ker_{\Z} B)
\oplus
(\ker_{\Z} A\otimes V)
\oplus
(U\otimes \ker_{\Z} B)
\oplus
(U\otimes V).
\end{equation}
Since $U, V$ are complement to $\ker_{\Z} A, \ker_{\Z} B$, respectively, no nonzero vector in $U\otimes V$ belongs to $\ker_{\Z}(A\otimes B)$, whereas other summands are contained in $\ker_{\Z}(A\otimes B)$. 
Therefore, we can decompose 
\begin{align}\label{eq:direct-decomposition}
\ker_{\Z}(A\otimes B) &=
(\ker_{\Z} A\otimes \ker_{\Z} B)
\oplus
(\ker_{\Z} A\otimes V)
\oplus
(U\otimes \ker_{\Z} B)\\
&= (\ker_{\Z} A\otimes \ker_{\Z} B) \oplus (\ker_{\Z} A\otimes V) + 
   (\ker_{\Z} A\otimes \ker_{\Z} B) \oplus (U\otimes \ker_{\Z} B)\\
&= \ker_{\Z} A\otimes \Z^{n_2}+\Z^{n_1}\otimes \ker_{\Z} B.
\end{align}
\end{proof}

The preceding lemmas isolate the two structural operations that govern the kernel of a Kronecker-stack configuration matrix. 
First, stacking row blocks imposes the intersection of these block-wise kernel conditions (Lemma~\ref{lem:kernel-stack}). 
Second, the kernel of each Kronecker-product block can be described factor-wise: a pure tensor is annihilated whenever at least one of its local factors is annihilated by the corresponding local matrix (Lemma~\ref{lem:kronecker}).

\begin{lemma}[Integer kernels of pre-defined factors]
\label{lem:kernel-factors}
Kernels of each factor used for a block matrix in Kronecker-product class can be explicitly derived.
\begin{enumerate}
    \item The kernel of an $n$-dimensional identity matrix is only the zero vector space. $\ker_{\Z} E_{n} = \{\bm{0}_{n}:=0\bm{1}_n\}$, which is a trivial vector.
    \item The kernel of the $n$-dimensional copy vector $\bm{1}_n^{\top}$ is spanned by a basis consisting of $n-1$ vectors $\delta_{k}\;(k=1,\ldots, n-1)\in\Z^n$ defined as:
    \begin{equation}
        \{\delta_{1},\delta_{2},\ldots, \delta_{n-1}\}:=\{
        e_{1(n)}-e_{n(n)},\;
        e_{2(n)}-e_{n(n)},\ldots, 
        e_{n-1(n)}-e_{n(n)}\}
    \end{equation} 
    where $e_{i(n)}$ denotes the $i$-th standard basis of $n$-dimensional vector space. We call $\delta_k$ an \emph{anti-copy} vector. 
    \item The kernel of a standard basis vector $e_{i(n)}^{\top}$ is the hyperplane orthogonal to it.
\end{enumerate}
In other words, for a factor $B$ with the column size of $n$, 
\begin{equation}
    \ker_{\Z}B = \begin{cases}
        \{\bm{0}\}&\mathrm{if}\;B=E_{n} \\
        \big\langle \{\delta_{i}\}_{i=1}^{n-1} \big\rangle_{\Z}&\mathrm{if}\;B=\bm{1}_{n}^\top \\
         \big\langle \{e_{j(n}\}_{j\neq i} \big\rangle_{\Z}&\mathrm{if}\;B=e_{i({n})}^\top \\
    \end{cases}.
\end{equation}
Here $\langle S \rangle_{\Z}$ denotes the $\Z$-span of a set $S$, that is, the set of all finite integer linear combinations of elements of $S$.
\end{lemma}

\begin{proof}
We omit the proofs for the first and third items, as they are standard results in linear algebra.
For the second, the kernel of $\bm{1}_n^{\top}$ has dimension $n-1$, and the $n-1$ vectors $\delta_{k}\;(k=1,\ldots, n-1)$ are clearly linearly independent, thus forming a basis for the kernel.
\end{proof}

\begin{proposition}[Kernel of selection stacks] 
\label{prop:selection_stack}
Let $A$ be a stacked matrix where each block uses a distinct standard basis vector $e_{i(n)}\in \mathbb{R}^n$ for $i=1,\ldots, n$ acting as a selector of a specific context (state), tensored with $n$ arbitrary matrix $B_i\in \R^{m_i \times n'}$ of the same column size $n'$:
\begin{equation}
    A = \begin{bmatrix}
        e_{1(n)}^{\top}\otimes B_1\\
        e_{2(n)}^{\top}\otimes B_2\\
        \vdots\\
        e_{n(n)}^{\top}\otimes B_{n}
    \end{bmatrix}\in \R^{\sum_im_i\times nn'}.
\end{equation}
Then, the kernel of $A$ decomposes cleanly into the direct sum of the individual kernels of $B_i$, embedded by the respective standard basis vectors:
\begin{equation}
    \ker A = \bigoplus_{i=1}^{n} \Big(e_{i(n)} \otimes \ker B_i \Big).
\end{equation}
\end{proposition}

\begin{proof}
Let $x\in\mathbb{R}^{nn'}$ be an arbitrary vector in the domain space. 
We can uniquely decompose $x$ along the standard basis $\{e_{i(n)}\}$ as:
$x = \sum_{i=1}^{n} e_{i(n)} \otimes x_i$.
That is, any vector $x$ can be uniquely split into $n$ consecutive chunks $x_1, \ldots, x_n \in \mathbb{R}^{n'}$ of equal length $n'$, where $x_i$ occupies the $i$-th block of coordinates of $x$.
Applying the $i$-th block of $A$ to $x$, the orthogonality of the standard basis vectors, i.e.,  $e_{i(n)}^\top e_{j(n)} = \delta_{ij}$, yields:
\begin{equation}\label{eq:selector-decomposition}
    \Big(e_{i(n)}^\top \otimes B_i\Big) x 
    = \Big(e_{i(n)}^\top \otimes B_i\Big)\left( \sum_{j=1}^{n} e_{j(n)}\otimes x_j\right)
    =\sum_{j=1}^{n} \Big(e_{i(n)}^\top e_{j(n)} \Big) \otimes B_i x_j = B_i x_i.
\end{equation}
Moreover, since $e_{1(n)}, \ldots, e_{n(n)}$ are linearly independent, the subspaces $e_{j(n)}\otimes \ker B_j \;(j=1,\ldots,n)$ intersect trivially, 
i.e., $\bigl(e_{j(n)}\otimes \ker B_j\bigr) \cap \bigl(e_{k(n)}\otimes \ker B_k\bigr) = \{0\}$ for $j\neq k$. 
Hence, the sum $\sum_{i=1}^{n}\bigl(e_{i(n)}\otimes \ker B_i\bigr) = \bigoplus_{i=1}^{n}\bigl(e_{i(n)}\otimes \ker B_i\bigr) $ is a direct sum under Definition~\ref{def:direct-sum}.

Since each $x_i$ is uniquely determined by reading off the $i$-th chunk of $x$, the condition $x \in \ker A$ forces $x_i \in \ker B_i$ for each $i$ independently (Eq.~\ref{eq:selector-decomposition}), and thus,  $x$ decomposes uniquely as $x = \sum_i e_{i(n)}\otimes x_i$ with $x_i\in\ker B_i$.
Thus, $x \in \ker A \iff Ax = \bm{0} \iff B_i x_i = 0 \text{ for all } i$. 
This implies $x_i \in \ker B_i$ for all $i$, completing the proof of the direct sum structure.
\end{proof}

\begin{remark}[Necessity of exhaustive selection in toric models]
\label{rem:exhaustive-selection}
Proposition~\ref{prop:selection_stack} highlights a crucial dichotomy in the kernel structure of the Kronecker-stack configuration matrix. 
While the all-one vector factor (\emph{copy vector}) in $\bm{1}^\top \otimes B$ \emph{copies} and couple together the states, yielding the \emph{anti-copy} vectors $\delta_i$ in the kernel, 
the standard basis factor in $e_{i(n)}^\top \otimes B_i$ acts as a selector that perfectly decouples the kernel into independent, branch-wise components.

In a toric model, the configuration matrix $A$ must contain the all-one vector $\bm{1}$ in its row space, which requires the column sums of $A$ to be constant (i.e., the model must be homogeneous). 
When introducing a selector $e_{i(n_k)}^{\top}$  with respect to the $k$-th factor to capture branch-wise or context-specific structures, we cannot arbitrarily omit any other state; doing so would result in zero-columns for the omitted states, directly violating the constant column sum property. 
Consequently, we must exhaustively use all standard basis vectors $i=1, \dots, n_k$ to cover the entire state space. 
\end{remark}

\subsection{Construction of Signatures}
Together with the explicit kernels of the three allowed local factors---identity, copy, and selector (Lemma~\ref{lem:kernel-factors})---this implies that the integer kernel of a Kronecker-stack matrix can be generated by local standard basis vectors and anti-copy vectors in a combinatorial way. 
In what follows, we make this observation precise: in Theorem~\ref{thm:factorwise_integer_kernel_generation} 
we show that the factor-wise admissible tensors constructed from these local vectors generate the entire integer kernel, 
after giving a series of definitions, lemmas, and corollaries.
This is the key step that makes the Graver basis, and hence the algebraic signature, effectively enumerable within the Kronecker-stack class.

\begin{definition}[Admissible set]
\label{def:admissible-set}
Let $A$ be a matrix in the Kronecker-stack class defined in Definition~\ref{def:kronekcer-class}: 
\begin{equation}
A =\begin{bmatrix}
        A^{(1)}\\\vdots\\A^{(R)}
\end{bmatrix},
\quad
A^{(r)}=\bigotimes_{k=1}^{m} B^{(r)}_{k},
\quad\mathrm{where}\;
B^{(r)}_{k}\in
\left\{
E_{n_k},\,
\mathbf{1}_{n_k}^{\top},\,
e_{i(n_k)\;(i\in [n_k])}^{\top}
\right\}.
\end{equation}

Define the set of free vectors and anti-copy vectors for its $k$-th mode ($k=1,\ldots, m)$ by
\begin{align}
\mathcal E_k:&=\big\{e_{1(n_k)},\ldots,e_{n_k(n_k)}\big\}\\
\mathcal D_k :&=\big\{\delta_{i(n_k)}:=e_{i(n_k)}-e_{n_k(n_k)}\mid 1\leq i\leq n_k-1\big\}
\end{align}
and let $\mathcal L_k:=\mathcal E_k \cup \mathcal D_k $.

For a tuple of local vectors chosen from $\mathcal{L}_k$ across modes, 
$(\bm u_1,\ldots,\bm u_m)\in \mathcal L_1\times\cdots\times\mathcal L_m$,
write
\begin{equation}
\bm u:=\bm u_1\otimes\cdots\otimes \bm u_m\in \Z^{N},
\end{equation}
where $N=\prod_{k=1}^{m} n_k$.
We call such a pure tensor $\bm u$ \emph{admissible} for $A$ if, for every row block $r=1,\ldots,R$, at least one local vector $\bm u_{\ell(r)}$ is annihilated by the corresponding local matrix:
\begin{equation}
\label{eq:global_annihilator_rule}
\forall r\in\{1,\ldots,R\}\;\left(
\exists\; \ell(r) \in\{1,\ldots,m\}
\;\text{such that}\;
B^{(r)}_{\ell(r)}\bm u_{\ell(r)}=0\right) .
\end{equation}

The collection of all admissible pure tensors of $A$ is called the  \emph{admissible set}, denoted as:
\begin{equation}
    C(A):=\Big\{\bm{u} \in \Z^N\mid \bm{u}=\bm u_1\otimes\cdots\otimes\bm u_m;\, \bm u_k\in \mathcal L_k\,\mathrm{for}\,k=1,\ldots, m;\, \bm{u}\mathrm{\,is\,admissible\;for}\;A\Big\}
\end{equation}
\end{definition}

Note that the annihilating mode $\ell(r)$ depends on $r$ and may differ across different row blocks.

\begin{corollary}\label{cor:intersection} 
Let $C(A^{(r)})$ be an admissible set for $A^{(r)}$ that is a single row-block matrix; $A^{(r)}=\bigotimes_{k=1}^m B_k^{(r)}$ and $A$ is a vertical stack of such matrices $A^{(1)},\ldots,A^{(R)}$, where the consisting factors $B_k^{(r)}$  in the same mode $k$ across $r$ is either, $E_{n_k}, \bm{1}_{n_k}^\top$ or $e_{i(n_k)}^\top$,
then the following intersection holds.
    \begin{equation}
        C(A) = \bigcap_{r=1}^R C(A^{(r)})
    \end{equation}
\end{corollary}
\begin{proof}
For $\bm u \in \bigcap_{r=1}^R C(A^{(r)})$, $\bm u \in C(A^{(r)})$ for all $r$.
By Definition~\ref{def:admissible-set}, $\bm{u}=\bm u_1\otimes \cdots\otimes\bm u_m \in C(A^{(r)})$  has at least one mode $\ell(r) \in \{1,\ldots m\} $ such that $B_{\ell(r)}^{(r)}\bm u_{\ell(r)}=\bm{0}$.
This holds for all $r$ since $\bm u \in C(A^{(r)})$ for all $r$, which is exactly the definition of $C(A)$.
\end{proof}

\begin{corollary}[Selector reduction for admissible sets]
\label{cor:selector-reduction}
Under the setting of Proposition~\ref{prop:selection_stack}, 
suppose mode $k$ contains a selector factor $e_{i(n_k)}^\top$ exhaustively
across all states $i=1,\ldots,n_k$ (as explained by
Remark~\ref{rem:exhaustive-selection}).
Let $B_i$ denote the Kronecker-stack matrix obtained by fixing the $k$-th
factor to $e_{i(n_k)}^\top$ and removing that mode.
Then the admissible set splits blockwise:
\begin{equation}
  C(A) = \bigsqcup_{i=1}^{n_k}\bigl\{e_{i(n_k)}\otimes u_i \mid u_i\in C(B_i)\bigr\}\label{eq:select_disjoint_union},
\end{equation}
where $\sqcup$ denotes a disjoint union.
\end{corollary}

\begin{proof}
By Proposition~\ref{prop:selection_stack}, any admissible vector (Definition~\ref{def:admissible-set}) can be expressed by a pure tensor $e_{i(n_k)}\otimes u_i$ for $u_i \in \ker_{\Z}B_i$.
Since distinct $\{e_{i(n_k)}\}$ are orthogonal, the \emph{strata}, i.e., the subsets of $\ker_{\mathbb{Z}}A$ sharing a common selector index $i$ for the $k$th mode,  share only the zero vector, 
so Eq.~\eqref{eq:select_disjoint_union} holds as a disjoint union.
\end{proof}

When several modes act as selectors, the same reduction applies mode by mode,
and the resulting strata are indexed by all combinations of their respective selector indices.
To express each block-wise kernel through a common integral basis in the
proof of Theorem~\ref{thm:factorwise_integer_kernel_generation}, we augment
the anti-copy vectors $\mathcal{D}_k$ of Definition~\ref{def:admissible-set}
by a single selector $e_{n_k(n_k)}$ into a basis of the whole mode-$k$ space.

\begin{lemma}[$\mathbb{Z}$-basis of $\mathbb{Z}^N$]
\label{lem:z-basis}
Let $\mathcal{B}_k := \mathcal{D}_k \cup \{e_{n_k(n_k)}\}$ for each mode $k$
as in Definition~\ref{def:admissible-set}, and let
\begin{equation}
  \mathcal{B} := \bigotimes_{k=1}^{m} \mathcal{B}_k
  = \Bigl\{ b_1 \otimes \cdots \otimes b_m \mid b_k \in \mathcal{B}_k \Bigr\}.
\end{equation}
Then $\mathcal{B}_k$ is a $\mathbb{Z}$-basis of $\mathbb{Z}^{n_k}$,
and $\mathcal{B}$ is a $\mathbb{Z}$-basis of $\mathbb{Z}^N$
with $N = \prod_{k=1}^{m} n_k$.
\end{lemma}

\begin{proof}
The transition matrix $T_k$ from $\mathcal{E}_k$ to $\mathcal{B}_k$ is lower
triangular with unit diagonal such as $T_3=\begin{bsmallmatrix}
1&0&0\\
0&1&0\\
-1&-1&1
\end{bsmallmatrix}$ in case of $n_k=3$, hence unimodular
\cite[p.~53]{aoki2012markov}, so $\mathcal{B}_k$ is a $\mathbb{Z}$-basis
of $\mathbb{Z}^{n_k}$.
The transition matrix of $\mathcal{B} = \bigotimes_k \mathcal{B}_k$
from the standard basis of $\mathbb{Z}^N$ is $\bigotimes_k T_k$,
whose determinant equals $\prod_k (\det T_k)^{N/n_k} = 1$.
Hence $\mathcal{B}$ is also a $\mathbb{Z}$-basis of $\mathbb{Z}^N$.
\end{proof}

\begin{corollary}[Single-block kernel generation: selector-free case]
\label{cor:single-block-basis}
Let $A^{(r)} = \bigotimes_{k=1}^{m} B_k^{(r)}$ be a single row-block matrix
with $B_k^{(r)} \in \{E_{n_k},\, \mathbf{1}_{n_k}^{\top}\}$ for all $k$.
Define $\mathcal{B}_r := C(A^{(r)}) \cap \mathcal{B}$, where $\mathcal{B}$
is the $\mathbb{Z}$-basis of Lemma~\ref{lem:z-basis}. Then
\begin{equation}
  \ker_{\mathbb{Z}} A^{(r)} = \langle \mathcal{B}_r \rangle_{\mathbb{Z}}.
\end{equation}
\end{corollary}

\begin{proof}
Applying Lemma~\ref{lem:kronecker} inductively, we have a decomposition 
$\ker_{\mathbb{Z}} A^{(r)}= \sum_{k=1}^{m} \mathbb{Z}^{n_1}\otimes\cdots\otimes
\ker_{\mathbb{Z}} B_k^{(r)} \otimes\cdots\otimes\mathbb{Z}^{n_m}$.
The sum over $k$ reflects that a pure tensor lies in $\ker_{\mathbb{Z}} A^{(r)}$
once at least one factor lies in $\ker_{\mathbb{Z}} B_k^{(r)}$.
For  $\bm b = \bm b_1\otimes\cdots\otimes \bm b_m \in \mathcal{B}$, this decomposition shows that
$\bm b \in \ker_{\mathbb{Z}} A^{(r)}$ if and only if at least one factor $\bm b_k$ lies in $\ker_{\mathbb{Z}} B_k^{(r)}$.
By Lemma~\ref{lem:kernel-factors}, $\bm b_k\in \ker_{\mathbb{Z}} B_k^{(r)}$ only when $B_k^{(r)}=\bm{1}_{n_k}^\top$ instead of $E_{n_k}$.
Thus, such $\bm b_k$ is an element of $\mathcal{D}_k$, and therefore $\bm b \in C(A^{(r)})$.

In sum, $\bm{b} \in \mathcal{B}$ that lies in $\ker_{\mathbb{Z}} A^{(r)}$ implies $\bm{b} \in C(A^{(r)})$, and the converse holds by Definition~\ref{def:admissible-set}:
$\mathcal{B} \cap \ker_{\mathbb{Z}}A^{(r)} = \mathcal{B}\cap C(A^{(r)}) = \mathcal{B}_r$ with the last equality as so defined.
Since $\mathcal{B}$ is a $\mathbb{Z}$-basis of $\mathbb{Z}^N$,
every $\bm{z}\in\ker_{\mathbb{Z}}A^{(r)}$ expands uniquely as
$\bm{z}=\sum_i c_i\bm{b}_i$, and $c_i=0$ for $\bm{b}_i\notin\mathcal{B}_r$.
Therefore $\ker_{\mathbb{Z}}A^{(r)}=\langle\mathcal{B}_r\rangle_{\mathbb{Z}}$.
\end{proof}

\begin{lemma}[Span--intersection commutativity on a common $\mathbb{Z}$-basis]
\label{lem:span-intersection}
Let $\mathcal{B} = \{\bm b_1, \dots, \bm b_N\}$ be a $\mathbb{Z}$-basis of $\mathbb{Z}^N$,
that is, a set of vectors such that every $\bm z \in \mathbb{Z}^N$ has a unique
representation
\begin{equation}
  \bm z = \sum_{i=1}^{N} c_i\, \bm b_i, \qquad c_i \in \mathbb{Z}.
\end{equation}
For $r = 1, \dots, R$, let $\mathcal{B}_r \subseteq \mathcal{B}$ be a subset, and
let $\langle \mathcal{B}_r \rangle_{\mathbb{Z}}$ be the subgroup of $\mathbb{Z}^N$ generated by $\mathcal{B}_r$.  
Then
\begin{equation}
  \bigcap_{r=1}^{R} \Big\langle \mathcal{B}_r \Big\rangle_{\Z}
  \;=\;
  \Big\langle \bigcap_{r=1}^{R} \mathcal{B}_r \Big\rangle_{\Z}.
\end{equation}
\end{lemma}

\begin{proof}
$(\Leftarrow) \;\bigcap_r \langle \mathcal{B}_r \rangle_{\mathbb{Z}} \supseteq \langle \bigcap_r \mathcal{B}_r \rangle_{\mathbb{Z}}$ is immediate.
Any $\bm z \in \langle \bigcap_r \mathcal{B}_r \rangle_{\mathbb{Z}}$ is an integer linear combination of basis vectors in $\bigcap_r \mathcal{B}_r \subseteq \mathcal{B}_q $ for every $q \in\{1,\ldots, R\}$, 
hence lies in $\langle \mathcal{B}_q \rangle_{\mathbb{Z}}$ for all $q$, i.e. $\bm{z}\in \bigcap_r \langle \mathcal{B}_r \rangle_{\mathbb{Z}}$.

For $(\Rightarrow)$, let $\bm z \in \bigcap_r \langle \mathcal{B}_r \rangle_{\mathbb{Z}}$ and write
$\bm z = \sum_{i=1}^{N} c_i \bm b_i$ for the unique coefficients $c_i \in \Z$.
For each $r$, since $\bm z \in \langle \mathcal{B}_r \rangle_{\mathbb{Z}}$,
we can write $\bm z = \sum_{b_i \in \mathcal{B}_r} c'_i \bm b_i$ for some $c'_i \in \Z$.
By uniqueness of the $\mathcal{B}$-expansion, $c_i = c'_i$ for
$\bm b_i \in \mathcal{B}_r$ and $c_i = 0$ for $\bm b_i \notin \mathcal{B}_r$.
Since this holds for every $r$, we have $c_i = 0$ whenever
$\bm b_i \notin \bigcap_r \mathcal{B}_r$,
hence $\bm z \in \langle \bigcap_r \mathcal{B}_r \rangle_{\mathbb{Z}}$.
\end{proof}

\begin{lemma}[Equivalence of $\Z$ span of admissible set and intersected basis]
\label{lem:C-cap-B}
With $\mathcal{B}_r := C(A^{(r)})\cap\mathcal{B}$ as in
Corollary~\ref{cor:single-block-basis},
\begin{equation}
  \langle C(A) \rangle_{\mathbb{Z}}
  = \Bigl\langle \bigcap_{r=1}^R \mathcal{B}_r \Bigr\rangle_{\mathbb{Z}}.
\end{equation}
\end{lemma}

\begin{proof}
By Corollary~\ref{cor:intersection}, $C(A) = \bigcap_r C(A^{(r)})$.
Restricting to $\mathcal{B}$,
$C(A)\cap\mathcal{B}
  = \Bigl(\bigcap_r C(A^{(r)})\Bigr)\cap\mathcal{B}
  = \bigcap_r \bigl(C(A^{(r)})\cap\mathcal{B}\bigr)
  = \bigcap_r \mathcal{B}_r$.
We span both sides to obtain $\langle C(A)\cap\mathcal{B}\rangle_\Z = \langle \bigcap_r \mathcal{B}_r\rangle_\Z$.
Since $C(A)\subseteq\mathcal{B}$ by construction, $C(A)=C(A)\cap\mathcal{B}$.
Hence $\langle C(A)\rangle_{\mathbb{Z}}
= \langle C(A)\cap\mathcal{B}\rangle_{\mathbb{Z}}
= \langle\bigcap_r \mathcal{B}_r\rangle_{\mathbb{Z}}$.
\end{proof}

\begin{theorem}[Factor-wise generation of the integer kernel]
\label{thm:factorwise_integer_kernel_generation}
Let \(A\) be a matrix in the Kronecker-stack class. 
Then, its admissible set generates the integer kernel of \(A\):
\begin{equation}
\left\langle C(A)\right\rangle_{\mathbb Z}
=
\ker_{\mathbb Z} A.
\end{equation}
\end{theorem}

\begin{proof}
If $A$ contains any selector factor $e_{i(n_k)}^\top$ in mode $k$,
Corollary~\ref{cor:selector-reduction}  decomposes $C(A)$ blockwise, and together with the direct-sum decomposition $\ker_\Z A = \bigoplus_i e_{i(n_k)} \otimes \ker_\Z B_i$ of Proposition~\ref{prop:selection_stack}, 
projecting onto each coordinate block gives $\langle C(A)\rangle_\Z=\ker_\Z A$ if and only if $\langle C(B_i)\rangle_\Z=\ker_\Z B_i$ for every $i$.
This reduces the claim to each stratum $B_i$ with one fewer selector mode; applying this recursively, 
it suffices to prove the theorem for $B_k^{(r)}\in\{E_{n_k},\mathbf{1}_{n_k}^\top\}$ for all $k$
and $r$.

In the selector-free case, the equality
$\ker_{\mathbb{Z}}A = \langle C(A)\rangle_{\mathbb{Z}}$
follows from the chain
\begin{equation}
  \ker_{\Z}A
  \overset{\text{Lem.\ref{lem:kernel-stack}}}{=}
  \bigcap_r \ker_{\Z}A^{(r)}
  \overset{\text{Cor.\ref{cor:single-block-basis}}}{=}
  \bigcap_r \langle\mathcal{B}_r\rangle_{\mathbb{Z}}
  \overset{\text{Lem.\ref{lem:span-intersection}}}{=}
  \Bigl\langle\bigcap_r\mathcal{B}_r\Bigr\rangle_{\mathbb{Z}}
  \overset{\text{Lem.\ref{lem:C-cap-B}}}{=}
  \langle C(A)\rangle_{\Z },
\end{equation}
where the kernel is successively expressed as an intersection of
block-wise kernels, rewritten in terms of the common $\mathbb{Z}$-basis
$\mathcal{B}$, commuted with the span, and finally identified with a $\Z$-span of the constructed admissible set $\langle C(A)\rangle_{\Z}$.
\end{proof}

Since $\langle C(A)\rangle_{\mathbb{Z}}=\ker_{\mathbb{Z}}A$
by Theorem~\ref{thm:factorwise_integer_kernel_generation},
the Graver basis of $A$ is obtained by taking the conformally indecomposable elements of $\Z$-span of $C(A)$, 
which satisfies $\langle\graver(A)\rangle_{\Z}=\ker_{\Z}A$.
The explicit construction is described in the next subsection.

\subsection{Examples of constructing signatures}
We demonstrate the combinatorial construction of the generating set, $C(A)$ via Theorem~\ref{thm:factorwise_integer_kernel_generation} for two example models with three binary variables $X_1, X_2, X_3$ where the number of states $n_k=2\;(k=1,2,3)$. 
For brevity, we denote the free basis candidates as $\mathcal{E} = \{e_1, e_2\}$ and the anti-copy candidate as $\mathcal{D} = \{\delta\}$, where $\delta := e_1 - e_2$.

\begin{example}[$2\times 2\times2$ Conditional independence model]
\label{ex:example_CI}
Consider the configuration matrix $A$ of Eq.~\eqref{eq:configuraion-CI},
representing the conditional independence $X_2 \perp\!\!\!\perp X_3 \mid X_1$ that is a hierarchical model with facets $\mathcal{F}=\{\{1,2\},\{1,3\}\}$:
\begin{equation}
\label{eq:configuraion-CI}
    A =\begin{bmatrix}
        A^{(1)} \\ A^{(2)}
    \end{bmatrix}
    =\begin{bmatrix}
        E_2 \otimes E_2 \otimes \bm{1}_2^\top\\
        E_2 \otimes \bm{1}_2^\top \otimes E_2\\
    \end{bmatrix}
\end{equation}
We form a tensor product  $\bm u = \bm u_1 \otimes \bm u_2 \otimes \bm u_3$ by selecting each local vector $\bm u_k \in \{e_1, e_2, \delta\}$ for $k=1, 2, 3$.
\begin{itemize}
    \item For $A^{(1)}= E_2 \otimes E_2 \otimes \bm{1}_2^\top$: Modes 1 and 2 act as identities and cannot annihilate the block. Thus, it strictly requires an annihilator in mode 3, thus we need to select $\bm u_3 = \delta$.
    \item For $A^{(2)}=E_2 \otimes \bm{1}_2^\top \otimes E_2$: Modes 1 and 3 cannot annihilate the block. Thus, it strictly requires an annihilator in mode 2, thus we need to have $\bm u_2 = \delta$.
\end{itemize}
Applying the intersection rule of Lemma~\ref{lem:kernel-stack}, a valid pure tensor must have $\bm u_2 =\bm u_3 = \delta$. 
Crucially, Mode 1 ($X_1$) interacts with the identity matrix $E_2$ in \emph{both} row blocks. Therefore, $\bm u_1$ is never required to act as an annihilator; it survives as a free component. Thus, we can choose any candidate for $\bm u_1 \in \{e_1, e_2, \delta\}$.
Thus the construction yields exactly three valid combinations for the admissible set:
\begin{equation}
    C(A) = \big\{ e_1 \otimes \delta \otimes \delta,\quad e_2 \otimes \delta \otimes \delta,\quad \delta \otimes \delta \otimes \delta \big\}
\end{equation}
Statistically, this perfectly aligns with the definition of conditional independence: we obtain one independent binomial constraint for the stratum $X_1=1$ by $e_1 \otimes \delta \otimes \delta$ and another for the stratum $X_1=2$ by $e_2 \otimes \delta \otimes \delta$. The third element is their algebraic difference, forming the complete admissible set.
\end{example}

\begin{example}[$2\times 2\times2$ Context-specific independence model]
\label{ex:example_CSI}

Consider the configuration matrix $A$ of Eq.~\eqref{eq:configuraion-CSI},
representing the context-specific independence $X_2 \perp\!\!\!\perp X_3$ only when $X_1=1$:
\begin{equation}
\label{eq:configuraion-CSI}
    A = \begin{bmatrix}
        A^{(1)} \\ A^{(2)} \\ A^{(3)}
    \end{bmatrix}
    = \begin{bmatrix}
        e_{1}^\top\otimes\bm{1}_2^\top\otimes E_2\\
        e_{1}^\top\otimes E_2\otimes\bm{1}_2^\top\\
        e_{2}^\top\otimes E_2\otimes E_2\\
    \end{bmatrix}
\end{equation}
We analyze the annihilator conditions:
\begin{itemize}
    \item $A^{(1)}$ will be annihilated if $\bm u_1 =e_2$ (orthogonal to $e_{1}^\top$) or $\bm u_2=\delta$.
    \item $A^{(2)}$ will be annihilated if $\bm u_1 =e_2$ (orthogonal to $e_{1}^\top$) or $\bm u_3=\delta$.
    \item $A^{(3)}$ will be annihilated only if $\bm u_1 = e_1$ (orthogonal to $e_{2}^\top$) since modes 2 and 3 cannot provide an annihilator for this block.
\end{itemize}
To satisfy the intersection rule of Lemma~\ref{lem:kernel-stack}, the tensor must annihilate all three blocks.
If we choose $\bm u_1 =e_2$, it fails to annihilate row block 3, thus we must choose $\bm u_1 = e_1$, forcing us to select $\bm u_2 = \delta$ (for block 1) and $\bm u_3 = \delta$ (for block 2).
This yields exactly one valid combination:
\begin{equation}
   C(A) = \big\{ e_1 \otimes \delta \otimes \delta \big\}
\end{equation}
Comparing this to the conditional independence model above, the structural rule elegantly filters out the constraint for $X_1=2$, leaving only the context-specific binomial.
\end{example}

\subsection{Computing a dictionary}
\label{sec:computing_dictionary}
By Theorem~\ref{thm:factorwise_integer_kernel_generation}, the admissible set $C(A)$ generates the integer kernel $\langle C(A)\rangle_{\mathbb Z}=\ker_{\mathbb Z}A$.
We therefore compute the Graver basis of $A$ as the Graver basis of the lattice generated by
$C(A)$:
\begin{equation}
    \operatorname{Gr}(A)
    :=
    \operatorname{Gr}\bigl(\langle C(A)\rangle_{\mathbb Z}\bigr).
\end{equation}
Since $\langle C(A)\rangle_{\mathbb Z}=\ker_{\mathbb Z}A$, this is exactly the Graver basis of
the integer kernel of $A$ that satisfies
\begin{equation}
\langle \operatorname{Gr}(A)\rangle_{\mathbb Z}
=
\langle C(A)\rangle_{\mathbb Z}
=
\ker_{\mathbb Z}A.
\end{equation}

The admissible set $C(A)$ is generally redundant: it generates the integer lattice, but its elements need not be conformally indecomposable. 
The Graver basis $\graver(A)$ is obtained from this lattice by retaining precisely the conformally indecomposable elements. 
Thus, in computation, we first construct $C(A)$ by the factor-wise admissibility rule, form the lattice $\langle C(A)\rangle_{\Z}$,
and then compute the Graver basis of this lattice.

\begin{example}[Computing a Graver basis of $2\times 2\times 2$ Conditional independence model]
We demonstrate the computation of the Graver basis for Example~\ref{ex:example_CI}. 
The core mechanism of filtering out conformally decomposable elements can be illustrated by the conditional independence model $X_2 \perp\!\!\perp X_3 \mid X_1$.
As we saw, the factor-wise construction yields the following admissible pure tensors in $C(A)$:
\begin{align*}
    \bm u_1 &= e_1 \otimes \delta \otimes \delta=(1,-1,-1,1,0,0,0,0)^\top \\
    \bm u_2 &= e_2 \otimes \delta \otimes \delta=(0,0,0,0,1,-1,-1,1)^\top , \\
    \bm u_3 &= \delta \otimes \delta \otimes \delta = (1,-1,-1,1,-1,1,1,-1)^\top.
\end{align*}
Recall the conformal decomposition $\delta = e_1 + (-e_2)$, where $e_1 \sqsubseteq \delta$ and $(-e_2) \sqsubseteq \delta$. 
This property allows us to explicitly decompose $\bm u_3$ as $\bm u_3 = \bm u_1 + (-\bm u_2)$. 
Because both $\bm u_1$ and $-\bm u_2$ are nonzero kernel elements conformally bounded by $\bm u_3$, the element $\bm u_3$ is conformally decomposable, not belong to $\graver(A)$.
Correspondingly, in terms of the polynomial ideal, this filtering correctly isolates the two quadratic binomials $p_{111}p_{122}-p_{112}p_{121}$ and $p_{211}p_{222}-p_{212}p_{221}$, successfully eliminating the redundant quartic binomial $p_{111}p_{122}p_{212}p_{221}-p_{112}p_{121}p_{211}p_{222}$ from the signature of the conditional independence model.
\end{example}

The preceding example illustrates a general pruning rule for passing from $C(A)$ to $\graver(A)$: 
if an admissible candidate contains an anti-copy factor $\delta_{s,t}=e_s-e_t$ in a non-annihilating mode $k$, its replacement by $e_s$ and by $e_t$ each yields an  admissible pure tensor, thus the candidate splits conformally into these fragments and is excluded before the final Graver-basis extraction.

\vspace{0.5cm}
For model selection, we apply this procedure not to a single model but to a finite family of candidate models. 
The resulting signatures can be arranged as a dictionary: each model is represented by the pattern of binomials that vanish on it. 
The following example illustrates this construction for the $2\times 2\times 2$ case.

\begin{example}[Constructing a dictionary for $2\times 2\times 2$ cases]
\label{eg:dictionaly_222}
As we explicitly list in Table~\ref{tab:hierarchical_models}, there are 9 hierarchical models in the $2\times 2\times 2$ case and additionally we consider 6 context-specific models, totally 15 models.
Applying the Graver-signature computation to those candidate models yields a common list of 13 binomial relations---12 quadratic binomials and one quartic binomial---from which the model-specific signatures are read:
\begin{equation}\label{eq:ideal_minors}
\begin{split}
    &m_1 = p_{111}p_{122}-p_{112}p_{121}, 
    \quad m_2 = p_{211}p_{222}-p_{212}p_{221},\\
    &m_3 = p_{111}p_{212}-p_{112}p_{211},
    \quad m_4 = p_{121}p_{222}-p_{122}p_{221},\\
    &m_5 = p_{111}p_{221}-p_{121}p_{211},
    \quad m_6 = p_{112}p_{222}-p_{122}p_{212},\\
    &m_7 = p_{111}p_{222}-p_{112}p_{221},
    \quad m_8 = p_{111}p_{222}-p_{121}p_{212},\\
    &m_9 = p_{111}p_{222}-p_{211}p_{122},
    \quad m_{10} = p_{112}p_{221}-p_{121}p_{212},\\
    &m_{11} = p_{112}p_{221}-p_{211}p_{122},
    \quad m_{12} = p_{121}p_{212}-p_{211}p_{122},\\
    &m_{13} = p_{111}p_{122}p_{212}p_{221}-p_{112}p_{121}p_{211}p_{222}.
\end{split}
\end{equation}

For each candidate model \(\mathcal M_i\), we define its binary signature vector
\begin{equation}
\bm{d}_i=(d_{i,1},\ldots,d_{i,13})^\top \in\{0,1\}^{13},
\quad
d_{i,j}:=
\begin{cases}
1 & \text{if } m_j \text{ vanishes}\,( \text{i.e.,}\; m_j=0)\, \text{on }\, \mathcal M_i,\\
0 & \text{otherwise}.
\end{cases}
\end{equation}
Stacking these vectors $\bm{d}_i$ as row gives the dictionary matrix
$D_{2\times2\times2}:=
\begin{bmatrix}
\bm{d}_0&
\cdots&
\bm{d}_{14}
\end{bmatrix}^\top
\in\{0,1\}^{15\times 13}$ for 15 models with 13 binomials displayed as a binary array with rows indexing models and columns indexing binomials:
{\footnotesize
\setlength{\arraycolsep}{3pt}
\begin{equation}
\label{eq:dictionary_222}
D_{2\times2\times2}
=
\begin{array}{c|cccccc|cccccc|c}
 & m_1 & m_2 & m_3 & m_4 & m_5 & m_6 & m_7 & m_8 & m_9 & m_{10} & m_{11} & m_{12} & m_{13} \\
\hline
\mathcal{M}_0\;(\textit{Saturated}) 
& 0 & 0 & 0 & 0 & 0 & 0 & 0 & 0 & 0 & 0 & 0 & 0 & 0 \\
\hline
\mathcal{M}_1\;(\textit{No3Way}) 
& 0 & 0 & 0 & 0 & 0 & 0 & 0 & 0 & 0 & 0 & 0 & 0 & 1 \\
\hline
\mathcal{M}_2\;(X_2\ind X_3\mid X_1) 
& 1 & 1 & 0 & 0 & 0 & 0 & 0 & 0 & 0 & 0 & 0 & 0 & 1 \\
\mathcal{M}_3\;(X_1\ind X_3\mid X_2) 
& 0 & 0 & 1 & 1 & 0 & 0 & 0 & 0 & 0 & 0 & 0 & 0 & 1 \\
\mathcal{M}_4\;(X_1\ind X_2\mid X_3) 
& 0 & 0 & 0 & 0 & 1 & 1 & 0 & 0 & 0 & 0 & 0 & 0 & 1 \\
\hline
\mathcal{M}_5\;((X_1,X_2)\ind X_3) 
& 1 & 1 & 1 & 1 & 0 & 0 & 1 & 0 & 0 & 0 & 0 & 1 & 1 \\
\mathcal{M}_6\;((X_1,X_3\ind X_2) 
& 1 & 1 & 0 & 0 & 1 & 1 & 0 & 1 & 0 & 0 & 1 & 0 & 1 \\
\mathcal{M}_7\;((X_2,X_3)\ind X_1) 
& 0 & 0 & 1 & 1 & 1 & 1 & 0 & 0 & 1 & 1 & 0 & 0 & 1 \\
\hline
\mathcal{M}_8\;(X_1\ind X_2 \ind X_3) 
& 1 & 1 & 1 & 1 & 1 & 1 & 1 & 1 & 1 & 1 & 1 & 1 & 1 \\
\hline
\mathcal{M}_9\;(X_2\ind X_3\mid X_1=0) 
& 1 & 0 & 0 & 0 & 0 & 0 & 0 & 0 & 0 & 0 & 0 & 0 & 0 \\
\mathcal{M}_{10}\;(X_2\ind X_3\mid X_1=1) 
& 0 & 1 & 0 & 0 & 0 & 0 & 0 & 0 & 0 & 0 & 0 & 0 & 0 \\
\mathcal{M}_{11}\;(X_1\ind X_3\mid X_2=0) 
& 0 & 0 & 1 & 0 & 0 & 0 & 0 & 0 & 0 & 0 & 0 & 0 & 0 \\
\mathcal{M}_{12}\;(X_1\ind X_3\mid X_2=1) 
& 0 & 0 & 0 & 1 & 0 & 0 & 0 & 0 & 0 & 0 & 0 & 0 & 0 \\
\mathcal{M}_{13}\;(X_1\ind X_2\mid X_3=0) 
& 0 & 0 & 0 & 0 & 1 & 0 & 0 & 0 & 0 & 0 & 0 & 0 & 0 \\
\mathcal{M}_{14}\;(X_1\ind X_2\mid X_3=1) 
& 0 & 0 & 0 & 0 & 0 & 1 & 0 & 0 & 0 & 0 & 0 & 0 & 0\\
\hline
\end{array}.
\end{equation}
}
Each row $\bm{d}_i^\top$ of $D_{2\times2\times2}$ is the model-specific algebraic signature of $\mathcal M_i$.
The model $\mathcal{M}_8$---where all 13 relations vanish---corresponds to complete independence denoted as $X_1 \ind X_2 \ind X_3$. 
The models $\mathcal{M}_5, \mathcal{M}_6$, and $\mathcal{M}_7$ correspond to joint independence (JI) such as $(X_1,X_2)\ind X_3$, 
while $\mathcal{M}_2, \mathcal{M}_3$, and $\mathcal{M}_4$ correspond to conditional independence (CI) such as  $X_2\ind X_3\mid X_1$ . 
$\mathcal{M}_1$ is the no-three-way interaction model, where all pairwise interactions exist but the three-way interaction vanishes and its sole constraint is the quartic binomial $m_{13}$.
$\mathcal{M}_0$---associated with the empty ideal ($\bm{d}_0=\bm{0}$)---corresponds to the saturated model. 
The remaining models $\mathcal{M}_i$ for $i=9, \ldots, 14$ correspond to context-specific independence (CSI) such as $X_2 \ind X_3 \mid X_1=0 $. 
\end{example}

The dictionary matrix therefore converts algebraic signatures into finite binary feature vectors. 
In Section~\ref{sec:learning}, an empirical vanishing pattern obtained from data is encoded as a binary vector $\hat{\bm{m}}\in\{0,1\}^{13}$, and model selection is performed by comparing $\hat{\bm{m}}$ with the rows $\bm{d}_i$ of $D_{2\times2\times2}$, for example by Hamming distance.

\begin{remark}
A signature for each model is not redundant since it consists of binomials corresponding to its Graver basis only.
However, the dictionary as a union of all models  may contain binomials that are not Graver elements of a specific model, but are Graver elements of the other models.
For example, $m_{13}$ enters the dictionary through the No3Way model $\mathcal M_1$, though it is a binomial corresponding to non-Graver basis element for other models such as $\mathcal M_2$ since it can be decomposed into a combination of $m_1$ and $m_2$.
Thus, a row vector of the dictionary matrix might contain $1$ for other binomials than those in a signature, however, 
it can uniquely identify a model as shown in Corollary~\ref{cor:dictionaly-identifiability}.
\end{remark}

\subsection{Copy/Anti-copy mechanism and MIC (minimum invariant constraints)}

The previous subsections defined the Kronecker-stack class and constructed its kernel block by block, assembling candidate kernel bases via Kronecker products and filtering them. 
Stepping back, this construction reveals a fundamentally combinatorial nature of the kernel, which in turn provides geometric insight into why the polynomial constraints take the form of binomials.

In the Kronecker-stack class, the three factor types play fundamentally different roles in the kernel. 
Identity factors $E_n$ have trivial kernels, and selector factors $e_i^\top$ merely restrict attention to a specific state. 
The entire nontrivial kernel structure therefore originates exclusively from copy factors $\bm{1}^\top$, whose kernel is spanned by the anti-copy vectors $\delta_i = e_{i(n)}- e_{n(n)}$ for $i=1,\ldots, n-1$. 
The key insight is that the kernel of the Kronecker-stack class is entirely characterized by anti-copy vectors: copy factors in the configuration matrix induce anti-copy vectors into the kernel, a correspondence we call the \emph{copy/anti-copy mechanism}. 
Since this structure is determined solely by the row space of the configuration matrix and not by the model parameters, it yields a parameter-invariant algebraic foundation for the binomial signatures in this class.

The copy/anti-copy mechanism admits a natural notion of minimality. 
For a single block configuration matrix $A$,  
consider a kernel vector $\bm{u} = \bigotimes_{k=1}^m \bm{u}_k \in \ker A$ that contains only one anti-copy factor $\delta_{i,j(n_k)} = e_{i(n_k)} - e_{j(n_k)}$ for some $k$ where $m$ is the number of factors and $e_{i(n_k)}$ is the $i$-th standard basis vector of the $n_k$-dimensional vector space. 
The corresponding constraint is a linear equality 
$p_{\phi_k(i)} = p_{\phi_k(j)}$, where $\phi_k(i)$ denotes the multi-index of the state determined by selecting $i$ in mode $k$,
which allows us to equate two states, thus trivial.
Therefore, non-trivial minimal constraint arises when two factors are anti-copy vectors: 
for distinct modes $s \neq t$ and states $i_s\neq j_s$, $i_t\neq j_t$, the kernel vector $\bm{u} = \bigotimes_{k=1}^m \bm{u}_k$ where
$\bm{u}_k= \delta_{i_s, j_s(n_s)}$ for $k=s$, 
$\bm{u}_k= \delta_{i_t, j_t(n_t)}$ for $k=t$, 
and $\bm{u}_k=e_{i_k(n_k)}$ otherwise.
Then, $\bm{u}$ corresponds to an irreducible quadratic binomial.
This motivates the following definition.

\begin{definition}[Minimum Invariant Constraint; MIC]
Let $A$ be a Kronecker-stack class matrix whose a single block has a form of $A^{(r)} = \bigotimes_{k=1}^m B_k $ with $B_k \in \{E_{n_k}, \bm{1}_{n_k}^\top, e_{i(n_k)}^\top\}$.
A kernel vector $\bm{u} = \bigotimes_{k=1}^m \bm{u}_k \in \ker A$ is called a \emph{Minimum Invariant Constraint (MIC)} if exactly two factors are anti-copy vectors: 
$\bm{u}_s = \delta_{i_s, j_s(n_s)}$ and $\bm{u}_t = \delta_{i_t, j_t(n_t)}$ for some $s \neq t \in [m]$, and $\bm{u}_k = e_{i_k(n_k)}$ otherwise.
The corresponding binomial is
\begin{equation}\label{eq:MIC-condition}
p_{\phi_{s,t}(i_s, i_t)} p_{\phi_{s,t}(j_s, j_t)} - p_{\phi_{s,t}(i_s, j_t)} p_{\phi_{s,t}(j_s, i_t)} = 0,
\end{equation}
where $\phi_{s,t}(i_s, i_t)$ denotes the multi-index of the state determined by selecting $i_s$ in mode $s$ and $i_t$ in mode $t$ (the remaining modes fixed by the selector factors).
\end{definition}

An MIC is the smallest local invariant unit arising from the copy/anti-copy mechanism of the Kronecker-stack class, and is thus conformally primitive with respect to this mechanism.
An MIC, if it exists for a given configuration matrix, belongs to its Graver basis since it is in the kernel and not conformally decomposable, assuming there are no trivially equivalent states.
Moreover, an MIC is a \emph{circuit}~\citep{herzog2018binomial} since the two anti-copy factors each contribute exactly two states to the support, while the remaining factors are standard basis vectors contributing a single fixed index, yielding a support of exactly four states that is minimal under inclusion by the same assumption. 
Whether an MIC is an \emph{indispensable} binomial depends on the model~\citep{aoki2012markov}. 
In a $2\times 2$ independence model, the toric ideal contains a single MIC 
$p_{11}p_{22} - p_{12}p_{21} = 0$, which is the unique generator and is 
therefore indispensable.
In contrast, in a $2\times 3$ independence model where the three MICs $p_{11}p_{22}-p_{12}p_{21}$, $p_{12}p_{23}-p_{13}p_{22}$, and $p_{11}p_{23}-p_{13}p_{21}$ connect states within a common fiber, the third is dispensable since it can be reached via additive composition of the first 
two.
It is noteworthy that an MIC implies minimality of degree only, i.e., quadratic, and does not refer to minimality with respect to a kernel basis, Graver basis, or Gr\"{o}bner basis for a given configuration matrix. 
This distinguishes MICs from existing algebraic primitives: although MICs coincide externally with circuits and Graver elements, their definition originates from the copy/anti-copy mechanism 
rather than from ideal-theoretic considerations, and their minimality is one of degree rather than set-membership.

\begin{remark}[Higher-order binomials as compositions of MICs]
\label{rem:mic-composition}
The MIC is the minimal quadratic invariant of the Kronecker-stack class. 
We advocate viewing MIC as the atomic building block of the kernel for the class:
every higher-order binomial can be assembled from MICs through one or both of two distinct 
mechanisms, multiplicative and additive.

\emph{Multiplicative composition (tensor product).}
When a Kronecker-stack matrix carries three or more copy factors across blocks,
its kernel admits vectors with three or more anti-copy factors, and the support expands multiplicatively. 
For example, in the No3Way interaction model of the $2\times2\times2$ case, 
the vector $\delta\otimes\delta\otimes\delta$ has support of size $2^3=8$ and yields 
the quartic binomial
\begin{equation}
p_{111}p_{122}p_{212}p_{221} - p_{112}p_{121}p_{211}p_{222}.
\end{equation}
The degree is governed solely by the number of anti-copy factors; 
This binomial is thus a direct multiplicative extension of a MIC dictated by the Kronecker-product structure of the configuration matrix.

\emph{Additive composition (integer combination).}
A second mechanism to generate a higher-degree binomial is an integer combination of kernel vectors, which corresponds at the level of probability coordinates to a multiplication of the corresponding MIC relations.
In particular, when two MICs share a probability coordinate and their integer 
combination cancels it, the additive composition produces a binomial of degree 
other than a power of 2.
Consider a $3\times3$ joint probability table in which the $2$-minor vanishes on the top-left and bottom-right $2\times2$ blocks, giving the two MICs
$\bm u_1 = (1,-1,0,\,-1,1,0,\,0,0,0)$ for $p_{11}p_{22}-p_{12}p_{21}$, and 
$\bm u_2 = (0,0,0,\,0,1,-1,\,0,-1,1)$ for $p_{22}p_{33}-p_{23}p_{32}$, 
which share the coordinate $p_{22}$. 
Their difference $\bm u_1-\bm u_2 = (1,-1,0,\,-1,0,1,\,0,1,-1)$
cancels $p_{22}$ and generates the cubic binomial
\begin{equation}
p_{11}p_{23}p_{32} - p_{12}p_{21}p_{33},
\end{equation}
equivalently obtained by eliminating $p_{22}$ between the relations 
$p_{11}p_{22}=p_{12}p_{21}$ and $p_{22}p_{33}=p_{23}p_{32}$. 
Although the operation is linear at the level of lattice vectors, the cancellation is valid only where the shared coordinate is positive, equivalently away from the structural zeros on the boundary of the probability simplex. 
The degree is thus controlled not by the Kronecker structure but by the combinatorics of shared coordinates. 
This phenomenon has a concrete manifestation in the structure of Markov bases: for the 
no-three-factor interaction model of $3\times3\times K$ contingency tables, the basic moves of degree 4 alone do not form a Markov basis, because applying them sequentially can force intermediate tables to contain negative entries~\citep[Chapter 9]{aoki2012markov}. 
Moves of degree 6, 7, 8, and 10 arise precisely to circumvent these boundary constraints, connecting fibers that the basic moves cannot reach without passing through the boundary. 
The additive composition of MICs is the algebraic mechanism underlying this phenomenon: 
shared coordinates act as pivots, and their cancellation produces the higher-degree moves that navigate around structural zeros.

Taken together, these mechanisms suggest a reframing of the algebraic invariants of the Kronecker-stack class and shed light on the complexity inherent to a computation of Markov basis~\citep{aoki2012markov,hara2012running}. 
Rather than treating higher-degree binomials as separate irreducible primitives, 
one may take the MIC as an atom of the theory and recover the others through tensor-product and integer composition. 
In the present section we use MICs only as the atomic quadratic coordinates underlying the dictionary construction; a systematic treatment of additive composition, structural zeros, and ideal saturation is left for a future research.
\end{remark}

\begin{remark}[MIC as an atom of invariance structure]
\label{rem:MIC-as-atom}
An MIC is a building block from which the invariance structure of the Kronecker-stack class is entirely assembled. 
This perspective unifies a broad range of familiar models under a single algebraic atom. 
In the $2\times2\times2$ case (Table~\ref{tab:hierarchical_models}), independence models, conditional independence models, and context-specific independence models are all expressed as collections of MIC atoms, 
while the no-three-way interaction model arises as their multiplicative assembly.
Beyond these named models, as seen in the previous Remark, additive assembly of MIC atoms captures structures that carry no independence name,  such as the $3\times 3$ table considered above, where $2$-minors vanish only on the top-left and bottom-right  submatrices.
This observation allows us to generalize the notion of independence to that of \emph{invariance}, characterizing a class of probability distributions whose parameter-invariant structure is assembled from MIC atoms through multiplicative and additive composition.
This generalization places the MIC at the center of the theory as the fundamental unit from which all algebraic invariants of the class are assembled, with independence and conditional independence as special cases of invariance.
This last example, however, already lies outside the Kronecker-stack class, so that a full account of the resulting class of invariance models, and of how MIC atoms
compose to generate it, is left for future study.
\end{remark}

Algebraic characterization in this section naturally motivates the empirical question addressed next: given a finite-sample empirical distribution, can we robustly recover the underlying algebraic signature and identify the correct model structure? 
Section~4 studies this operational problem through proof-of-concept numerical experiments.

\section{Model Selection by Algebraic Signature Matching}
\label{sec:learning}

\subsection{Proof-of-Concept Setting}
\label{sec:POC_setting}
As established in the previous sections, the bijective relationship between statistical models and their algebraic signatures enables structural model selection directly from empirical data. 
In this section, we demonstrate this learnability through a proof-of-concept (PoC) experiment utilizing a tractable family of probability models. 
We propose \emph{Algebraic Signature Analysis (ASA)}, 
a structural learning procedure operating for confirmatory model selection via signature matching, which is called ASA-M in this section.
Here, the task for learning is to identify the specific model structure that generates an observed probability distribution characterized by certain inherent algebraic invariants. 
This formulation fundamentally contrasts with conventional machine learning paradigms, where parameter optimization is typically performed under a pre-assumed, fixed model structure. 
Our primary goal is to operationalize a core principle of algebraic statistics—namely, that the parameter-invariant structure within a probability distribution can uniquely dictate the model class. 

To mitigate the combinatorial complexity inherent in large hypothesis spaces, this PoC focuses on a manageable model scale,  leaving the comprehensive discussion of computational complexity and acceleration strategies to Section~\ref{sec:discussion}. 
To this end, we restrict our attention to $K \times K \times K$ models consisting of three discrete variables $X_1, X_2, X_3$, each taking $K$ states, where $K \in \{2, 3, 4, 5\}$. 
By leveraging the framework of hierarchical log-linear models, we can systematically describe any interaction among variables using the facets of a simplicial complex. 
As shown in Table~\ref{tab:hierarchical_models}, there exist 9 distinct facets and the corresponding models and additionally $3K$ context-specific independence with $K$ variations for each of three variables, thus, totally $9 + 3K$ distinct models.

According to the procedure described in Example~\ref{eg:dictionaly_222} in Section~\ref{sec:computing_dictionary}, 
we explicitly construct the algebraic signature for each candidate model as a set of vanishing binomials, then integrate them into a dictionary by taking their union. 
By evaluating the empirical values of the binomials contained in this dictionary, we perform a signature matching procedure to identify the model whose theoretical vanishing pattern best aligns with the data. 
For this PoC verification, we employ a brute-force inspection of signatures, which involves enumerating all candidate binomials, including local $2\times 2$ minors, within a given probability tensor to assess their vanishing conditions.
For the specific case of $K=2$, the resulting dictionary consists of 12 local $2\times2$ minors (quadratic binomials) and a single quartic binomial, as detailed in the system of defining equations of Eq.~\eqref{eq:ideal_minors} in the dictionary of Eq.~\eqref{eq:dictionary_222}.

\subsection{Threshold Selection to Determine Constraint Satisfaction}
\label{subsec:threshold}
\subsubsection{Sampling errors and pseudo-independence}
To determine whether a given empirical probability distribution satisfies the algebraic signature of a specific model, we evaluate the polynomial values of the vanishing binomials comprising that signature. 
Theoretically, if a probability vector $\bm{p}$ is generated by a model that possesses a vanishing binomial $m(\bm{p})$, the relation $m(\bm{p}) = 0$ holds strictly. 
In empirical settings, however, finite-sample fluctuations inevitably cause the evaluated binomial $m(\hat{\bm{p}})$ of an empirical distribution $\hat{\bm{p}}$ to deviate from zero. 
To accommodate this sampling noise, we introduce a decision rule governed by a tolerance threshold $\tau > 0$: if $|m(\hat{\bm{p}})| \leq \tau$, the constraint is determined to be satisfied (i.e., vanishing); otherwise, it is rejected. 
When sampling error is sufficiently large, this thresholding procedure inherently introduces Type~I errors (false positive), where a truly vanishing invariant is incorrectly classified as non-vanishing due to stochastic noise.

Furthermore, the detection framework encounters Type~II errors (false negative), which are prominently driven by a geometric phenomenon termed ``pseudo-independence.'' 
In the standard toric formulation, the joint probability vector is determined by the parameter vector $\bm{\theta}$ via the relation $p_i = \frac{1}{Z} \exp(\bm{\theta}^\top \bm{a}_i)$, where $\bm{a}_i$ denotes the $i$-th column of the configuration matrix $A$ and $Z$ is the normalizing partition function. 
Even when the underlying true model is saturated, a coincidental alignment of the parameter values within $\bm{\theta}$ can position the resulting probability distribution arbitrarily close to an independence manifold (e.g., Segre variety~\citep{sullivant2023algebraic}). 
More precisely, if a linear combination of parameters approaches zero---for instance, if a parameter vector $\bm{\theta}\in \R^4$ and an \emph{alignment vector} $\bm{u} = (1, -1, -1, 1)^\top$ as an independence constraint satisfies $\bm{\theta}^\top \bm{u} \approx 0$, then the distribution exhibits pseudo-independence. 
In the extreme case where $\bm{\theta}^\top \bm{u} = 0$, the saturated model becomes structurally indistinguishable from the independence model, thereby generating false negatives regardless of sample size.

Therefore, the calibration of a threshold $\tau$ requires resolving a fundamental trade-off: 
a higher $\tau$ mitigates sampling-induced false positives, whereas a lower threshold suppresses false negatives arising from pseudo-independence.
The optimal value should be selected to minimize the total error rate, thereby maximizing the classification accuracy of the model selection procedure. 

\subsubsection{
Statistical test on normalized binomial invariants
\label{sec-stats}}

To formalize this optimization as a tractable toy problem, we investigate the simplest non-trivial scenario: a $2 \times 2$ joint distribution where the probability vector $\bm{p}$ resides in the 3-dimensional probability simplex $\Delta_3$. 
In this setting, the structural learning task reduces to a binary hypothesis testing problem between the independence model (hypothesis $H_0$) and the saturated model (hypothesis $H_1$), each associated with a configuration matrix $A_0$ or $A_1$, respectively, 
\begin{equation}
    A_0 = \begin{bmatrix}
        1 & 1 & 0 & 0 \\
        0 & 0 & 1 & 1 \\
        1 & 0 & 1 & 0 \\
        0 & 1 & 0 & 1
    \end{bmatrix}, \qquad
    A_1 = \begin{bmatrix}
        1 & 0 & 0 & 0 \\
        0 & 1 & 0 & 0 \\
        0 & 0 & 1 & 0 \\
        0 & 0 & 0 & 1
    \end{bmatrix}.
\end{equation}
Under $H_0$, the statistical structure is fully characterized by a single vanishing binomial constraint:
$m(\bm{p}) = p_{11}p_{22} - p_{12}p_{21}$.
For this analysis, the log-linear parameters $\bm{\theta} \in [0,1]^4$ of the true generating distribution are assumed to be drawn from a uniform prior distribution.

To ensure that the threshold can be applied uniformly across any local $2 \times 2$ minors regardless of the absolute probability scales of the four constituent cells beyond this toy setting, we introduce a scale-homogeneous variant. 
Specifically, instead of evaluating the raw value of $m(\bm{p})$, we define the {\it normalized binomial invariant} $\tilde{m}(\bm{p})$ as
\begin{equation}
\label{eq-normalized-binomial}
    \tilde{m}(\bm{p}) := \frac{p_{11}p_{22} - p_{12}p_{21}}{p_{11}p_{22} + p_{12}p_{21}}.
\end{equation}
Unlike the unnormalized binomial invariant $p_{11}p_{22} - p_{12}p_{21}$, whose magnitude scales with the absolute probabilities of the four constituent cells, 
the normalized invariant $\tilde{m}(\bm{p})$ is homogeneous of degree zero and hence invariant under a common rescaling of the cells.
A single threshold on $\tilde{m}(\bm{p})$ is comparable across minors of widely differing probability magnitudes, 
independently of any distributional assumption on $\bm{p}$.
Appendix~\ref{sec-stat-normalized} discusses this preference for the normalized invariant in the elementary single-ratio case.

\subsubsection{Monte Carlo calibration of thresholds}

The scale-invariance of $\tilde{m}(\bm{p})$ makes a single threshold $\tau$ meaningful across minors, but does not by itself fix its value.
We therefore set $\tau$ so that it balances the Type~I and Type~II error
estimated from empirical data, as follows. 
In empirical settings, the observed count vector sampled from the true distribution $\bm{p}$ follows a multinomial distribution, yielding an empirical probability vector $\Hat{\bm{p}}$. 
To optimize the threshold $\tau$, we minimize the expected total error rate by balancing the inherent trade-off between Type~I and Type~II errors. 
Let $\alpha(\tau)$ and $\beta(\tau)$ denote the Type~I (false positive) and Type~II (false negative) error rates, respectively, both averaged over the prior distribution of $\bm{\theta}$:
\begin{align}
    \alpha(\tau) &= \int_{\bm{\theta}} \Pr\left( |\Tilde{m}(\hat{\bm{p}})| > \tau \mid \bm{\theta}, H_0 \right) f(\bm{\theta}) d\bm{\theta} \\
    \beta(\tau)  &= \int_{\bm{\theta}} \Pr\left( |\Tilde{m}(\hat{\bm{p}})| \leq \tau \mid \bm{\theta}, H_1 \right) f(\bm{\theta}) d\bm{\theta},
\end{align}
where $f(\bm{\theta})$ is the prior density. 
The total risk objective function $R(\tau)$ to be minimized is formulated as:
\begin{equation} \label{eq:errors}
    R(\tau) = \pi_0 \alpha(\tau) + \pi_1 \beta(\tau),
\end{equation}
where $\pi_0$ and $\pi_1$ denote the prior probabilities of models $H_0$ and $H_1$, respectively. 
The optimal threshold $\tau^*$ satisfies the first-order condition
$$\pi_0 \frac{d\alpha}{d\tau}\bigg|_{\tau^*} = -\pi_1 \frac{d\beta}{d\tau}\bigg|_{\tau^*},$$
equating the prior-weighted marginal reduction in Type~I error to the corresponding marginal increase in Type~II error.

By the Central Limit Theorem, the empirical probability vector $\hat{\bm{p}}$ can be asymptotically approximated by a multivariate Gaussian distribution around the true probability vector $\bm{p}$. 
However, the normalized invariant $\tilde{m}(\hat{\bm{p}})$ is a nonlinear function of $\hat{\bm{p}}$, and therefore it does not itself follow a Gaussian distribution in finite samples. 
Although a first-order Gaussian approximation can be obtained by the delta method under regularity conditions, our threshold calibration does not rely on this approximation. 
Instead, we directly estimate the finite-sample distributions of $\tilde{m}(\hat{\bm{p}})$ under $H_0$ and $H_1$ by Monte Carlo simulation (See Appendix~\ref{app:mc_threshold_calibration} for details).
Under a symmetric prior ($\pi_0 = \pi_1 = 0.5$) and a uniform distribution $\bm{\theta} \sim \text{Uniform}[0,1]^4$, the numerically optimized thresholds $\tau^*$ can be computed as in Table~\ref{tab:tau_theory}.
\begin{table}[htbp]
    \centering
    \caption{Monte Carlo calibrated thresholds $\tau^*$ for the normalized invariant $\tilde{m}$.}
    \label{tab:tau_theory}
    \begin{tabular}{lrrrrr}
    \hline
    \textbf{Models} & \textbf{\# Cells} $C$ & $N_{\mathrm{eff}}/N$ & $N = 1,000$ & $N = 10,000$ & $N = 100,000$ \\
    \hline
    $2{\times}2{\times}2$ & 8   & 0.500 & 0.156 & 0.062 & 0.024 \\
    $3{\times}3{\times}3$ & 27  & 0.148 & 0.232 & 0.104 & 0.042 \\
    $4{\times}4{\times}4$ & 64  & 0.062 & 0.301 & 0.139 & 0.059 \\
    $5{\times}5{\times}5$ & 125 & 0.032 & 0.351 & 0.179 & 0.072 \\
    \hline
    \end{tabular}
\end{table}
Since each $2 \times 2$ minor evaluates exactly four joint states, the classification accuracy is fundamentally governed not by the total sample size $N$, but by the \emph{local effective sample size} $N_{\mathrm{eff}}$, defined as $N_{\mathrm{eff}} = N \times (4/C)$, where $C = K^3$ denotes the total number of cells in the tensor in Table~\ref{tab:tau_theory}. 

In practice, practitioners can operationalize the threshold selection via two distinct paradigms: 
either theoretically utilizing those simulated values (\emph{theoretical values}), or empirically optimizing $\tau$ through a data-driven \emph{grid search} in a learning step  (\emph{grid search values}).

\subsection{Proof-of-Concept Experiment}
\label{subsec:signature-experiment}

\subsubsection{Objectives}
The objective of this proof-of-concept (PoC) experiment is to empirically demonstrate the structural learning capability of the proposed method in identifying the underlying true probability model from finite data. 
Within this framework, structural learning is operationalized as a signature matching problem, wherein the empirical signature vector extracted from noisy data is validated against the pre-computed theoretical signatures of all candidate models. 
Furthermore, the experiment aims to show that the proposed method offers a substantial advantage in computational cost over conventional likelihood-based model selection frameworks, such as the Akaike Information Criterion (AIC)~\citep{akaike1974} and the Bayesian Information Criterion (BIC)~\citep{schwarz1978}.

\subsubsection{Synthetic Data Generation}
Synthetic datasets for the evaluation are generated through a systematic two-step simulation process. 
First, for each candidate model $\mathcal{M}_i$ within the $K \times K \times K$ hypothesis space, a true log-linear parameter vector $\bm{\theta} \in [0,1]^d$ is sampled from a uniform prior distribution. 
The exact joint probability vector $\bm{p}$ is then analytically computed via the toric parametrization dictated by the configuration matrix $A_i$ associated with that specific model. 
Second, to emulate empirical fluctuations, observed count tensors are simulated by sampling from a multinomial distribution governed by the true probability vector $\bm{p}$. 
We evaluate total sample sizes of $N = 10^3, 10^4, \text{ and } 10^5$, and additionally incorporate the exact probability vector $\bm{p}$ itself to represent the asymptotic empirical distribution under an infinite sample size ($N = \infty$). 
For each individual model class, we generate 500 independent realizations as training data and 100 independent realizations as test data, yielding a total test set size of $100 \times (9+3K)$ samples. 
The training dataset is utilized exclusively to optimize the threshold $\tau^*$ via a data-driven grid search. 
Conversely, when deploying the theoretically derived threshold or evaluating the conventional maximum-likelihood-based criteria (AIC/BIC), model selection is performed directly on the test dataset without any prior training phase.

\subsubsection{Identification Procedure}
Given an empirical distribution $\hat{\bm{p}}$ derived from the observed sample counts, the model identification process is executed through the following sequential steps:
\begin{enumerate}
    \item \textbf{Invariant Evaluation and Binarization}: The polynomial values for all binomials contained in the pre-constructed unified dictionary $\mathcal{D}$ are evaluated. For a given test sample $\hat{\bm{p}}$, the normalized invariant values $\tilde{m}_j(\hat{\bm{p}})$ are computed for each binomial in the dictionary and binarized into an empirical signature vector $\hat{\bm{v}}$, where the $j$-th entry is mapped to $1$ if $|\tilde{m}_j(\hat{\bm{p}})| \leq \tau^*$ (indicating constraint satisfaction within noise tolerances) and $0$ otherwise.
    \item \textbf{Signature Matching and Distance Minimization}: The resulting binary vector $\hat{\bm{v}}$ is systematically matched against the theoretical signature vectors $\bm{\bm{d}}_i$ of all candidate models within the hypothesis class. If the empirical vector exhibits an exact match with the theoretical signature of a specific model, that model is uniquely selected.
    \item \textbf{Distance Metrics and Tie-Breaking}: In scenarios where no candidate model provides an exact binary match, we employ the Hamming distance as the primary metric to quantify the proximity between the empirical and theoretical binary signature patterns. 
    Furthermore, if a tie occurs among multiple competing models under the Hamming metric, the Euclidean distance between the raw, unbinarized empirical binomial values  and their theoretical targets (i.e., zero or one) for a model $\mathcal{M}_i$ is utilized as a deterministic tie-breaking criterion to uniquely identify the optimal model.

\end{enumerate}
The optimal tolerance threshold $\tau^*$ is operationalized either by the theoretical values obtained by a Monte Carlo simulation in Table~\ref{tab:tau_theory} or by empirical optimization via a grid search over the learning dataset. 
To benchmark the identification performance, the proposed ASA-M is compared directly with the standard likelihood-based framework using AIC and BIC scores calculated via maximum likelihood estimation as the standard method~\citep{sullivant2023algebraic}.

\subsubsection{Evaluation Metric}
The primary evaluation criterion is the model identification accuracy, defined as the proportion of test samples for which the selected model exactly agrees with the true generating model class.

\subsection{Results}
\label{subsec:learning-results}

\begin{table}[t]
\centering
\caption{%
    Model selection accuracy : test cases $100\times (9+3K)$ for $K=2,3,4,5$}
\label{tab:accuracy}
\setlength{\tabcolsep}{6pt}
\begin{tabular}{lr|cc|ccc}
\toprule
Model family & Sample size ($N$)
      & ASA\,(GS)
      & ASA\,(TH)
      & ML
      & AIC
      & BIC \\
\midrule
\multirow{4}{*}{$K=2\;(2{\times}2{\times}2)$}
 & $\infty$ (true probability)     & ---    & \textbf{1.000} & ---   & ---   & ---   \\
 & $1{,}000$         & 0.368  & \textbf{0.379} & 0.067 & 0.362 & 0.324 \\
 & $10{,}000$        & \textbf{0.693} & \textbf{0.693} & 0.069 & 0.604 & 0.664 \\
 & $100{,}000$       & 0.860  & 0.856          & 0.069 & 0.705 & \textbf{0.864} \\
\addlinespace[4pt]
\hline
\addlinespace[4pt]
\multirow{4}{*}{$K=3\;(3{\times}3{\times}3)$}
 & $\infty$ (true probability)         & ---    & \textbf{1.000} & ---   & ---   & ---   \\
 & $1{,}000$         & 0.542  & 0.394          & 0.056 & \textbf{0.782} & 0.472 \\
 & $10{,}000$        & 0.965  & 0.937          & 0.056 & 0.917 & \textbf{0.988} \\
 & $100{,}000$       & 0.997  & 0.998          & 0.056 & 0.927 & \textbf{1.000} \\
\addlinespace[4pt]
\hline
\addlinespace[4pt]
\multirow{4}{*}{$K=4\;(4{\times}4{\times}4)$}
 & $\infty$ (true probability)          & ---    & \textbf{1.000} & ---   & ---   & ---   \\
 & $1{,}000$         & 0.495  & 0.248          & 0.048 & \textbf{0.898} & 0.281 \\
 & $10{,}000$        & 0.985 & 0.967  & 0.048 & 0.977 & \textbf{0.986} \\
 & $100{,}000$       & \textbf{1.000} & \textbf{1.000} & 0.048 & 0.984 & \textbf{1.000} \\
\addlinespace[4pt]
\hline
\addlinespace[4pt]
\multirow{4}{*}{$K=5\;(5{\times}5{\times}5)$}
 & $\infty$  (true probability)       & ---    & \textbf{1.000} & ---   & ---   & ---   \\
 & $1{,}000$         & 0.287  & 0.222          & 0.042 & \textbf{0.920} & 0.058 \\
 & $10{,}000$        & \textbf{0.995} & 0.923  & 0.042 & 0.992 & 0.958 \\
 & $100{,}000$       & \textbf{1.000} & \textbf{1.000} & 0.042 & 0.995 & \textbf{1.000} \\
\bottomrule
\end{tabular}
\label{tab:result_accuracy}
\end{table}

\subsubsection{Accuracy}

Table~\ref{tab:result_accuracy} reports the model selection accuracy achieved by the evaluated methods across different tensor dimensions and sample sizes. 
Here, we take two methods aim at the optimal threshold $\tau^{*}$:
ASA~(GS) estimates it by a data-driven one-dimensional grid search on the held-out training trials, 
while ASA~(TH) uses the theoretically derived value in Table~\ref{tab:tau_theory}.
As a degenerate baseline, ML represents unconstrained maximum likelihood selection without any model complexity penalty. 
Bold text indicates the highest classification accuracy within each row.

Under the asymptotic condition of an infinite sample size ($N=\infty$), the model selection accuracy for the true probability distribution reaches $1.000$ across all tested dimensions ($K=2, 3, 4, 5$), signifying that the theoretical binomial constraints perfectly characterize and isolate the underlying statistical models. 
This perfect recovery provides empirical and operational validation for our core algebraic framework, as the true joint probability vector $\bm{p}$ generated by a model must strictly satisfy the algebraic invariants defined within the kernel of its corresponding configuration matrix.

In the finite-sample scenarios encompassing 12 distinct empirical conditions, the proposed ASA methods achieved the highest performance among all tested frameworks in 5 scenarios: $(K,N)=(2,10^3), (2,10^4), (4, 10^5), (5,10^4), \text{ and } (5, 10^5)$. 
In these settings, the invariant-matching paradigm successfully outperformed both conventional information-theoretic frameworks, namely AIC and BIC. 
In contrast, the conventional criteria exhibited sample-size-dependent performance profiles; 
AIC achieved the highest accuracy in the smallest sample size regime ($N=10^3$), 
whereas BIC demonstrated superior accuracy in the larger sample size regimes ($N=10^4$ and $N=10^5$). 
Crucially, the unconstrained ML baseline degenerated to a chance-level performance. 
Because ML lacks a regularization penalty for model complexity, straight likelihood maximization systematically favors the unconstrained parameter space, consistently selecting the saturated model $\mathcal{M}_0$ due to its maximal degrees of freedom. 
Consequently, the accuracy of the ML baseline is strictly bounded by the prior reciprocal share of the model space; for instance, it yields a base rate of exactly $1/15 \approx 0.067$ in the $K=2$ setting.

\subsubsection{Thresholds: Theoretical Optimization vs. Grid Search}

Figure~\ref{fig:grid-search-combined} compares the two alternative approaches for determining the optimized threshold: the values derived from the theoretical simulation (denoted by star markers) and those obtained empirically via the data-driven grid search (denoted by cross markers). 
As illustrated in the left panel ($K=2$), the theoretical threshold in cases of smaller $K$ closely approximates the optimal threshold selected through the exhaustive grid search, demonstrating high alignment across the evaluated sample sizes.

\begin{figure}[htbp]
    \centering
    \begin{subfigure}[b]{0.48\linewidth}
        \centering
        \includegraphics[width=\linewidth]{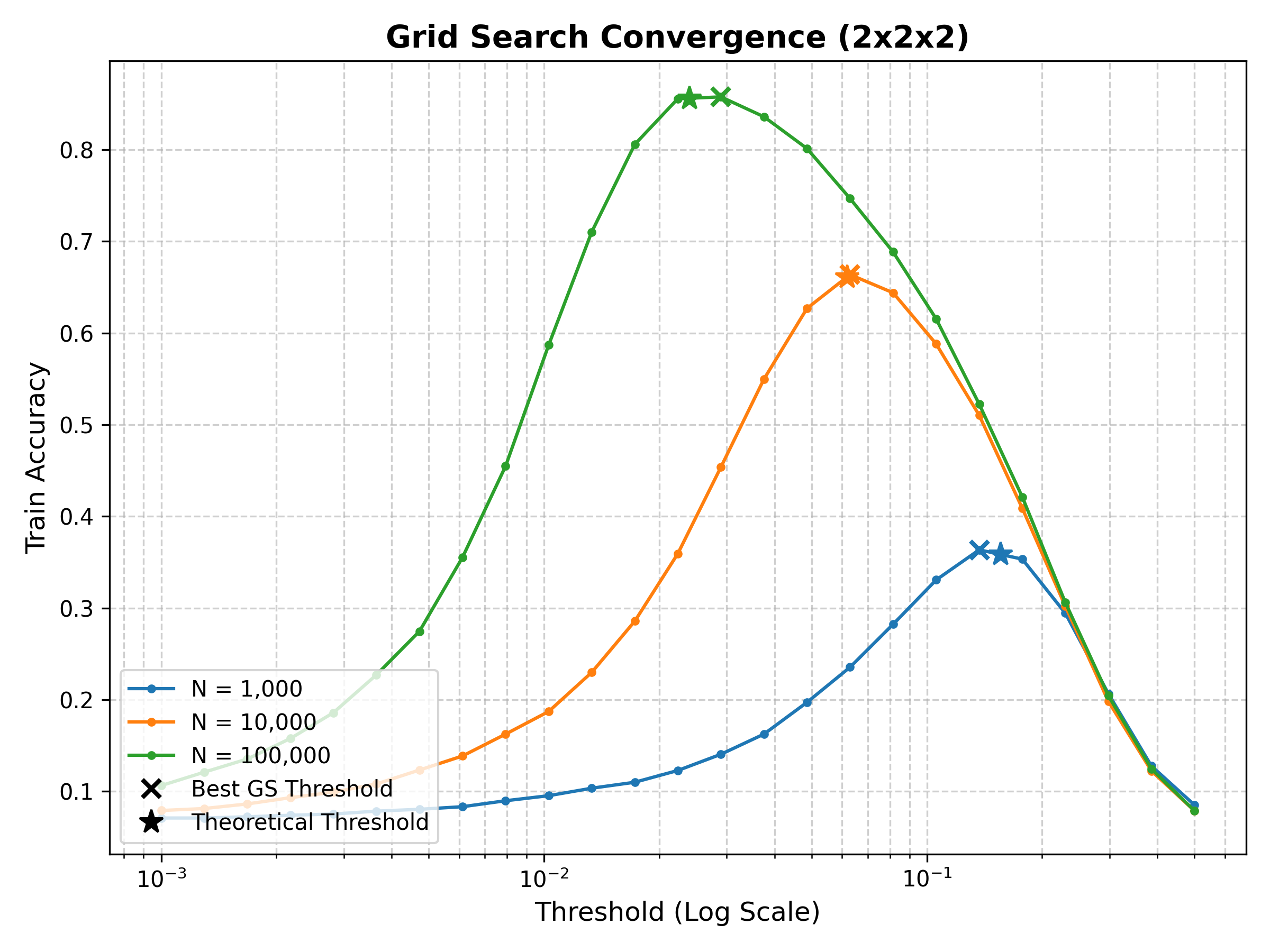}
        \caption{$2\times 2\times 2$ case.}
        \label{subfig:grid_2x2x2}
    \end{subfigure}
    \hfill 
    \begin{subfigure}[b]{0.48\linewidth}
        \centering
        \includegraphics[width=\linewidth]{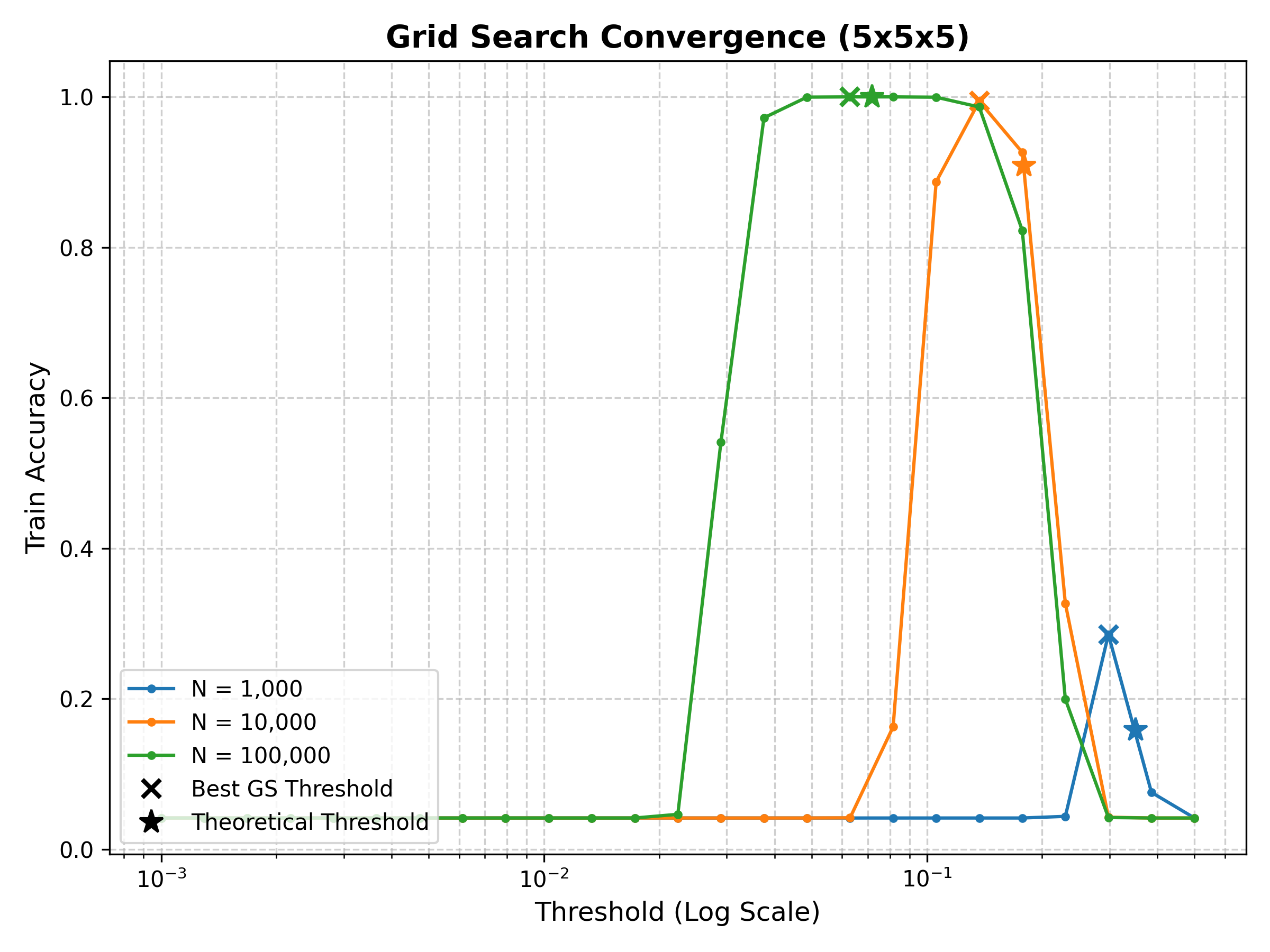}
        \caption{$5\times 5\times 5$ case.}
        \label{subfig:grid_5x5x5}
    \end{subfigure}
    
    \caption{Comparison of the optimized thresholds obtained via the theoretical scaling formula (star markers) and the empirical grid search (cross markers).}
    \label{fig:grid-search-combined}
\end{figure}

In contrast, for models with more states such as $K=5$, combined with smaller sample sizes, 
the theoretical threshold exhibits a noticeable deviation from the empirically selected grid-search threshold.
The impact of sampling noise is larger due to the stronger effect of smaller effective sample sizes.
This discrepancy can be attributed to the underlying assumptions of the analytical derivation; 
the theoretical formula assumes a symmetric (50:50) prior distribution between a single independent hypothesis and a saturated hypothesis. 
In the actual multi-model testing environment with a larger state size $K$, however, the effective prior proportion among the multiple overlapping invariants may deviate from the assumed prior distribution, thereby introducing a structural bias to the values computed by the theoretical formula.

\subsubsection{Error Analysis of Invariant Judgment}

To investigate the source of misclassifications in the ASA classifier,
we analyze judgment errors at the level of individual binomial invariants
rather than at the level of model selection.
Specifically, for each pair of a model $\mathcal{M}_i$ and an invariant $m_j$,
we examine whether the threshold-based binary judgment---vanishing ($\hat{v}_{i,j}=1$)
or non-vanishing ($\hat{v}_{i,j}=0$)---agrees with the true signature entry $\bm{d}_{i,j}$.

We focus on the $2\times2\times2$ case, which comprises 15 models and 13 binomial
invariants, yielding a $15\times13$ matrix of judgment opportunities.
Two types of error are distinguished:
\begin{itemize}
    \item \textbf{Type I error}: $\mathcal{D}_{i,j}=1$ (truly vanishing) but judged as
          non-vanishing ($\hat{v}_{i,j}=0$).
    \item \textbf{Type II error}: $\mathcal{D}_{i,j}=0$ (truly non-vanishing) but judged as
          vanishing ($\hat{v}_{i,j}=1$).
\end{itemize}
\begin{figure}[h]
    \centering
    \includegraphics[width=1.0\linewidth]{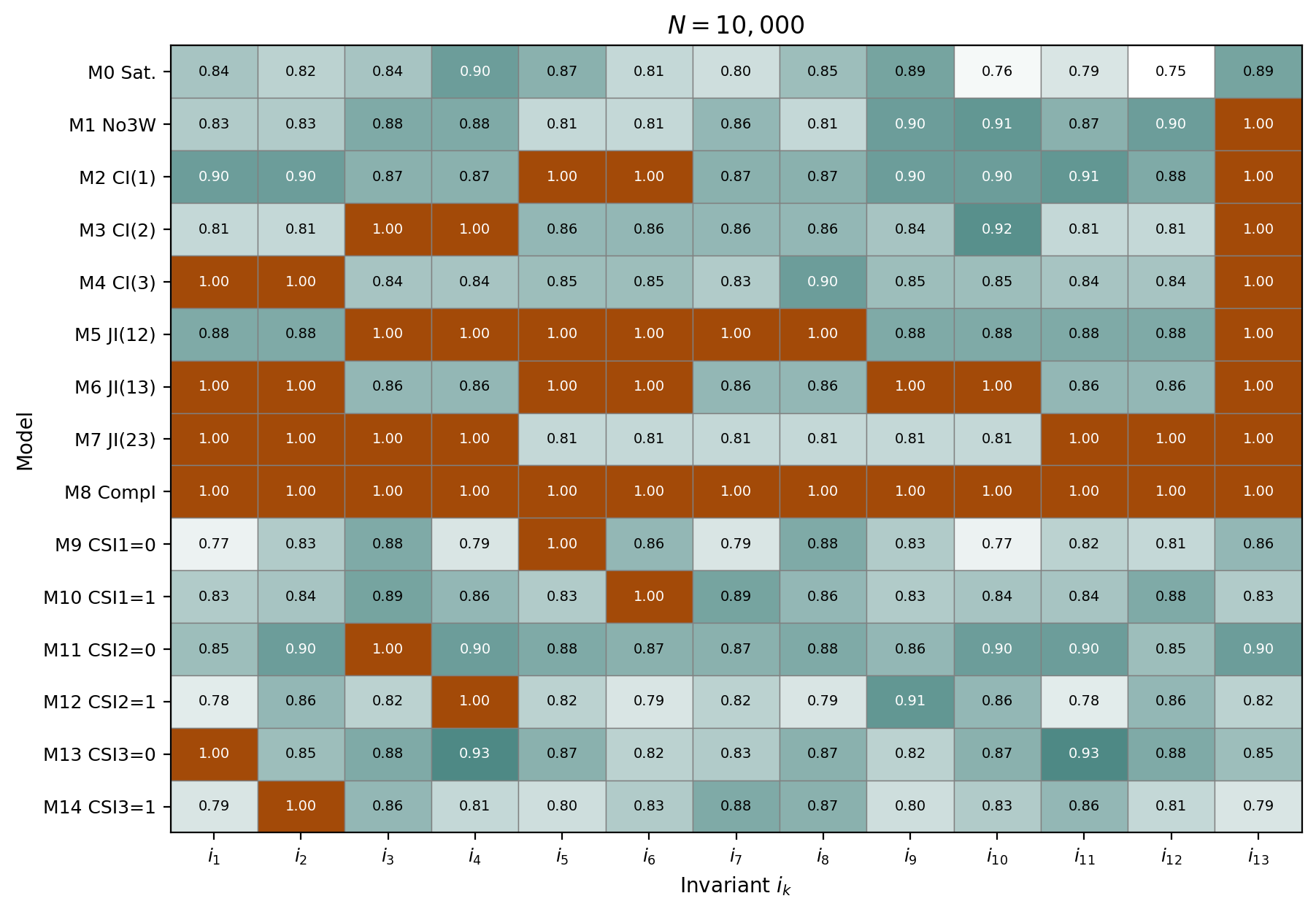}
    \caption{Confusion matrix on invariant judgment accuracy per 
model $M_i$ and binomial $i_k$ for the $2\times2\times2$ case at $N=10^4$, using threshold $\tau=0.062$.
Amber: signature $=1$ (constraint), intensity encoding Type~I error risk;
teal: signature $=0$ (free), intensity encoding Type~II error risk.
Intensity scales from the minimum (white) to $1.00$ in each color.}
    \label{fig:error-analysis}
\end{figure}
Figure~\ref{fig:error-analysis} displays the judgment accuracy for each cell
$(i,j)$, computed over 100 test samples per model.
Amber cells correspond to entries where $\mathcal{D}_{i,j}=1$ (Type I error risk);
teal cells correspond to entries where $\mathcal{D}_{i,j}=0$ (Type II error risk).

The results reveal a clear asymmetry between the two error types.
All amber cells achieve accuracy 1.00 across all sample sizes, indicating that Type I errors are entirely absent:
whenever an invariant truly vanishes, the threshold correctly identifies it as such.
In contrast, teal cells show accuracy in the range $[0.75, 0.93]$ at $N=10{,}000$,
confirming that Type II errors occur at a non-negligible rate.
That is, certain non-vanishing invariants are incorrectly judged as vanishing, and it is precisely these judgment failures that propagate into model misclassifications.
This pattern is consistent with the phenomenon of pseudo-independence,
in which the parameter vector $\bm \theta$ is approximately aligned with the kernel of the binomial map, causing a genuinely non-vanishing binomial to take a value close to zero and thereby fall below the threshold $\tau^*$.
The prevalence of Type II errors thus reflects an intrinsic geometric limitation of the threshold-based classifier—its inability to separate genuine independence from pseudo-independence when $\bm\theta$ lies near the kernel—rather than a deficiency of the threshold calibration procedure.

\subsubsection{Computational Efficiency}
A key practical advantage of the proposed ASA classifier over conventional likelihood-based criteria is its computational simplicity. Model selection via information criteria such as AIC or BIC  requires, for each candidate model, computing the maximum likelihood estimate via an iterative numerical optimizer. Each optimization iteration involves evaluating the log-partition function, computing analytical gradients, and updating the parameter estimates; consequently, the number of iterations required for convergence is inherently data-dependent and cannot be bounded \emph{a priori}.
In contrast, the binomial features in the ASA dictionary are derived from the model family alone and can therefore be precomputed once offline, entirely decoupling the per-sample classification cost from any MLE computation.

Table~\ref{tab:inference_time} reports the wall-clock inference time per sample for ASA and BIC across the four model families considered in this study ($K \in \{2, 3, 4, 5\}$). The ASA framework is between \num{153} and \num{874} times faster than BIC, with the most pronounced relative advantage observed at $K = 3$ (\num{874}$\times$). Although the absolute inference time for ASA scales with the number of binomial features (ranging from \SI{0.015}{\milli\second} at $K=2$ to \SI{0.521}{\milli\second} at $K=5$), this growth reflects the expanding expressiveness of the feature dictionary rather than any iterative algorithmic complexity, and it remains highly amenable to further acceleration via batched matrix operations.

\begin{table}[htbp]
    \centering
    \caption{Inference time for ASA and BIC across model sizes $K \in \{2,3,4,5\}$.}
    \label{tab:inference_time}
    \begin{tabular}{c|rrr|S[table-format=1.3]|S[table-format=2.3]|S[table-format=3.0]}
    \toprule
    $K$ & \# features & \# models & \# dim & {ASA [ms]} & {BIC [ms]} & {Speedup ($\times$)} \\
    \midrule
    2 &     13   & 15 &   8 & 0.015 &  9.285 & 630 \\
    3 &    270   & 18 &  27 & 0.022 & 19.084 & 874 \\
    4 &  1,944   & 21 &  64 & 0.064 & 31.815 & 498 \\
    5 &  8,500   & 24 & 125 & 0.521 & 79.437 & 153 \\
    \bottomrule
    \end{tabular}
\end{table}

\vspace{0.5cm}
In conclusion, the experimental results demonstrate that ASA model selection successfully identifies probabilistic structures through parameter-invariant algebraic constraints. 
Unlike likelihood-based criteria, which evaluate optimized likelihood values balanced against penalties for model dimensionality, the proposed framework directly assesses the empirical vanishing patterns of binomial constraints. 
Thus, these empirical evaluations support the central thesis that ASA signatures can function as computable algebraic features for structural learning in finite discrete probability tensors, achieving performance comparable to established likelihood-based criteria.

\section{A Linguistic Application of ASA-D}
\label{sec:linguistic}
We demonstrate that ASA (Algebraic Signature Analysis) extends beyond synthetic probability tensors to large-scale empirical data, by applying ASA-D---its discovery mode for unknown geometric structure---to a corpus of natural-language sentences. 
Our aim here is not to test a linguistic hypothesis, but to show concretely how ASA-D is deployed on large-scale data, and that the mathematical insights, machinery, and theorems developed above are applicable to a broad new range of targets in machine learning, including computational linguistics and natural language processing. 
The central finding is that rank-one sub-tensors satisfying MIC conditions are recoverable from noisy, high-dimensional language data, and that the recovered structures are linguistically interpretable:
ASA-D groups together words that are mutually exchangeable within a shared context, purely from algebraic constraints on corpus frequencies and with no semantic supervision. 
While the significance of this may not be immediately apparent from an algebraic-statistics standpoint, it is of real practical value for computational linguistics, where the unsupervised discovery of such locally coherent semantic structure is a long-standing concern. 
The geometric structure of word representations in language models such as transformer-based large language models~\cite{vaswani2017} is an active area of study, and ASA
offers a potentially valuable tool in this context. 
As far as we could ascertain, no previous studies have applied algebraic statistics to natural-language data at corpus scale.

\subsection{Experiment}
\label{subsec:setup}

\subsubsection{Language as a toric model.}
A language model assigns a probability to a sentence, that is, to the co-occurrence of its constituent words. 
Restricting attention to subject--verb--object sentences, the probability of a sentence is the joint probability $p_{s,v,o}$ of its three syntactic arguments $s \in \mathcal{S}$,
$v \in \mathcal{V}$, $o \in \mathcal{O}$, 
so that a language model over SVO triples is exactly a distribution over the three random variables $\mathcal{S}\times\mathcal{V}\times\mathcal{O}$.
We model this joint distribution by a toric model defined by Eq.~\eqref{eq:toric} in Section~\ref{sec:primer}.
Under the standard linguistic assumption of compositionality---that the probability of a sentence is governed by interactions among its constituent words rather than by holistic sentence-level parameters---the configuration matrix $A$ decomposes as a Kronecker product of per-word factors and therefore belongs to the Kronecker-stack class of Section~\ref{sec:stack-definition}. 
The toric model and the MIC-based structural learning of the preceding sections thus apply
to language modeling directly, with no modification.

\subsubsection{Corpus and tensor construction.}
We extracted subject--verb--object triples from the De\-pen\-den\-cy-\hspace{0pt}parsed Common Crawl corpus \citep[DepCC;][]{panchenko-etal-2018-building}, the largest dependency-annotated
English corpus available, and accumulated trigram counts into an integer tensor $T \in \mathbb{Z}_{\geq 0}^{500\times 500\times 500}$,
retaining the most frequent subject, verb, and object lemmas.
Two well-known properties of corpus data require attention before the algebraic constraints can be read off~\citep{manning99foundations}. 
First, word frequencies are extremely skewed, following the approximately power-law behavior known as Zipf's law, so we work on a logarithmic scale.
Second, the co-occurrence tensor is extremely sparse, and most empty cells correspond to grammatically or semantically inadmissible word combinations that must not be mistaken for algebraic vanishing; 
we therefore treat such cells as structural zeros via a positive-association filter. 
These two steps yield a preprocessed tensor, whose construction is described in the Appendix~\ref{app:preproc} and~\ref{app:mic}.

\subsubsection{Identifying rank-one structure via MICs.}
ASA-D then operates directly on this tensor. 
For each $2\times2$ minor---a quadruple of cells on two of the three
axes---it evaluates the binomial of the MIC condition  (Eq.~\ref{eq:MIC-condition}) to
determine whether that minor is an MIC.
A block of cells on which the MICs jointly vanish is exactly a rank-one submatrix.
Detecting MICs is thus equivalent to extracting rank-one structure locally and cell by
cell.
Since examining every quadruple of the full $500\times500\times 500$ tensor is computationally prohibitive,
ASA-D proceeds greedily: 
first, fixing one axis at $o\in \mathcal{O}$, it forms the $\mathcal{S}\times\mathcal{V}$ slice at that $o$ and, from the subject and verb words tied together by common vanishing MICs, constructs a rank-one submatrix $S'\times V'$,
then collects every $o'\in\mathcal{O}$ whose slice shares this same rank-one submatrix, yielding a rank-one sub-tensor $T[S',V',O']$ with $S'\subseteq\mathcal{S}$, $V'\subseteq\mathcal{V}$,
$O'\subseteq\mathcal{O}$. 
The full assembly procedure is detailed in Appendix~\ref{app:assembly}.

\subsection{Results}
\label{subsec:results}
\subsubsection{Recovery of rank-one sub-tensors.}
A single ASA-D run on the full $500\times500\times 500$ tensor recovered
$1{,}889$ rank-one sub-tensors passing the quality thresholds, with a
median rank-one fit of $97\%$ (the rank-one fit measures how closely a
sub-tensor approximates an exact rank-one tensor, and equals $100\%$ when
the MIC condition holds exactly; see Appendix~\ref{app:metric-definition} for
all metric definitions).
This recovery of near-exact rank-one structure from noisy corpus counts
is the central empirical finding of ASA-D's application.
Without positing any model in advance, ASA-D discovers local invariant constraints---MICs---as atoms directly from the data, and recombines them into rank-one structures;
the structures so obtained are local, confined to small subsets $S'\times V'\times O'$ rather than spanning the global tensor.
Since structural zeros are removed by the filtering step (see Appendix~\ref{app:mic} for details), 
the recovered sub-tensors are non-sparse, and their consistently high rank-one fit confirms that the extracted structures are substantively rank-one, rather than artifacts of sparse support.
This indicates that MICs are indeed discoverable in real language distributions and that ASA-D is effective at recovering them, confirming that the parameter-invariant algebraic structure posited by our theory is a detectable property of real data, and not merely a
feature of the controlled synthetic models of Section~\ref{sec:learning}.

\subsubsection{Representative examples.}
Table~\ref{tab:linguistic} lists five of the recovered sub-tensors,
together with the three metrics discussed below. 
Each row is a self-contained semantic domain, obtained with no semantic supervision and
purely from the algebraic vanishing pattern of corpus frequencies. 
The first row, for instance, collects research-agent nouns such as \emph{study}, \emph{author}, and \emph{research} as subjects, investigative verbs such as \emph{find} and \emph{examine}, and causal-explanatory nouns such as \emph{effect}, \emph{cause}, and \emph{factor} as objects, and is naturally read as an academic-research domain; we call such a recovered sub-tensor a \emph{frame}.
The remaining rows admit analogous readings as electoral-politics, news/media, software-release, and sports-statistics frames.
Recovering algebraic structure from the corpus through MICs thus yields word groups that are linguistically coherent and interpretable, despite the absence of embeddings, supervision, or any prior specification of the categories. 
This suggests that algebraic statistics can extract linguistic structure from language data in the form of algebraic constraints on corpus frequencies.

\begin{table}[htbp]
\centering
\caption{Five selected rank-one SVO sub-tensors with the corresponding words.
Due to limited space object words are truncated to the first 10 items and only non-functional words are shown.
The \emph{Domain} label is assigned post hoc by the authors.
The three metric definitions are given in Appendix~\ref{app:metric-definition}.}
\label{tab:linguistic}
\small
\setlength{\tabcolsep}{3pt}
\begin{tabular}{>{\raggedright\arraybackslash}p{1.6cm} |
                >{\raggedright\arraybackslash}p{2.5cm}
                >{\raggedright\arraybackslash}p{2.8cm}
                >{\raggedright\arraybackslash}p{4.0cm} | r r r}
\toprule
Domain\newline\footnotesize(Size)
  & $S$ (subjects)
  & $V$ (verbs)
  & $O$ (objects)
  & \makecell{Rank-1\\fit (\%)}
  & \makecell{Lift\\(times)}
  & \makecell{NMF\\recall (\%)} \\
\midrule
Academic research\newline\footnotesize $3\times4\times5$
  & study, author, research
  & find, rule, examine, address
  & effect, cause, factor,\newline influence, theory
  & 98.5 & 210 & 21.7 \\[4pt]
\hline
Electoral politics\newline\footnotesize $3\times3\times8$
  & bush, obama,\newline clinton
  & win, lose, capture
  & percent, house, state, \newline majority, nomination, \newline presidency, vote
  & 97.7 & 417 & 18.1 \\[4pt]
\hline
News media\newline\footnotesize $3\times3\times41$
  & times, news, post
  & publish, feature,\newline run
  & story, item, image, \newline account, kind, number
  & 97.4 & 104 & 19.0 \\[4pt]
\hline
Software tech\newline\footnotesize $3\times5\times26$
  & version, release, update
  & introduce, deliver,\newline include, contain, add
  & system, change, level, \newline support, message, list, \newline version
  & 97.3 & 101 & 17.2 \\[4pt]
\hline
Sports stats\newline\footnotesize $4\times3\times19$
  & brown, johnson,\newline jones, williams
  & contribute, add, record
  & point, goal, performance, \newline hit, score, pair, yard, \newline catch
  & 97.1 & 220 & 21.1 \\
\bottomrule
\end{tabular}
\end{table}

\subsubsection{Algebraic content of a frame.}
Every recovered sub-tensor encodes a vanishing MIC, a binomial relation among the joint probabilities that holds on the logarithmic scale introduced above.
For the \emph{electoral-politics} sub-tensor, one such condition reads
\begin{equation}
\label{eq:ling_mic}
  p(\textit{obama},\textit{win},\textit{vote})\,
  p(\textit{clinton},\textit{lose},\textit{vote})
  -
  p(\textit{clinton},\textit{win},\textit{vote})\,
  p(\textit{obama},\textit{lose},\textit{vote}) = 0.
\end{equation}
Algebraically, Eq.~\eqref{eq:ling_mic} is a relation among corpus frequencies
and nothing more.
Yet it coincides with a purely linguistic judgment: 
\textit{obama} and \textit{clinton}, both U.S.\ presidents, are interchangeable as subjects,
and \textit{win} and \textit{lose}, antonyms that share the
same syntactic and semantic frame, are interchangeable as verbs, 
in the sense that each substitution remains both grammatical and contextually acceptable.
We call these properties \emph{local symmetry} for the algebraic relation and \emph{exchangeability} for the linguistic relation, respectively. 
The coincidence of these two notions is not trivial.
It indicates that a symmetry of language is reflected in the algebraic structure of its probability distribution, and it is exactly this correspondence that
ASA-D detects, with no recourse to linguistic knowledge.

\subsubsection{Comparison with non-negative matrix factorization.}
Finding rank-one structure in a matrix or tensor is usually achieved
through factorization techniques such as non-negative matrix
factorization \citep[NMF;][]{lee1999learning}. 
In contrast, ASA-D finds rank-one submatrices locally, by detecting the vanishing of individual minors. 
We compare the two on the same data: for each extracted sub-tensor we measure how much of the ASA-D group a single NMF component recovers, 
a quantity we call NMF recall (defined in Appendix~\ref{app:metric-definition}). 
Across the five frames it stays low, between $17.2\%$ and $21.7\%$ (Table~\ref{tab:linguistic}), which implies structural differences between the two approaches. 
First, NMF is driven by variance: high-magnitude, high-frequency cells dominate its Frobenius-norm objective, so algebraically constrained but low-frequency blocks remain invisible to it. 
Second, NMF operates globally, seeking a rank-$K$ approximation of the entire $S\times V$ matrix, whereas ASA-D operates locally, growing structure upward from individual vanishing minors without reference to global frequency gradients.
The high rank-one fit of the resulting sub-tensors shows that this local, bottom-up search recovers compact and sharply defined structures that the global decomposition does not resolve.

\subsection{Significance}
\label{subsec:ling_significance}

The $1{,}889$ rank-one sub-tensors recovered from the corpus constitute,
to our knowledge, the first large-scale demonstration that algebraic invariant
structure is detectable as an interpretable linguistic phenomenon. 
A natural question is whether $1{,}889$ is a meaningful count or simply the
number of rank-one structures that ASA-D returns on any tensor of this format,
independent of its content.
To answer this, we constructed a permutation null by randomly permuting all
cells within each object slice, preserving the per-slice value distribution while destroying the SVO association pattern.
Applying the identical pipeline to the permuted tensor reduced the recovery from
$1{,}889$ to $0$ rank-one sub-tensors under the same filtering criteria.
Thus, the method is capable of identifying the rank-one structure of the joint co-occurrence distribution itself. 
We interpret this as evidence that lexical knowledge inherently resides in distributions as geometric structure such as the Segre variety, and that such structure is algebraically recoverable. Experiment details are given in Appendix~\ref{app:permutation}.

Interpretable rank-one structures have already been observed in bigram co-occurrence matrices~\citep{maeda-etal-2024-decomposing}, 
where they explain geometric properties of word vector representations such as
word2vec~\citep{mikolov2013distributed} and GloVe~\citep{pennington2014glove}. 
Word vector representations are a foundational technology of natural language processing and underlie transformer-based LLMs~\citep{vaswani2017}. 
The internal mechanisms of LLMs are an active area of study today, and algebraically constrained low-dimensional structure embedded in high-dimensional ambient spaces
connects directly to the manifold hypothesis in machine learning. 
Our proposed methods and the results of their preliminary application to natural-language data are therefore expected to contribute to the broader program of mechanistic interpretability of LLMs~\citep{sharkey2025openproblemsmechanisticinterpretability} through an algebraic-geometric characterization of such structure.

Two aspects of the present application are intentionally limited in scope and point to future work. 
First, ASA-D as applied here uses only quadratic MICs;
higher-degree binomials encoding interactions among multiple local invariants are not examined, though language may contain complex higher-order interactions; 
the full combinatorial power of MIC composition, where each recovered sub-tensor serves as the atomic unit from which higher-order invariant structures are assembled via the mechanisms of Remarks 3.25--3.26, is the subject of ongoing work.
Second, the search is greedy owing to the exponential growth in the number of candidate minors;
a more exhaustive yet tractable exploration of the full cell space, and the development of an efficient algorithm for this purpose, are directions we are actively pursuing.

\section{Discussion}
\label{sec:discussion}

In this study, the Kronecker-stack class was introduced because its factor-wise structure permits an explicit description of the integer kernel via the copy/anti-copy mechanism (Theorem~\ref{thm:factorwise_integer_kernel_generation}).
This restricts expressivity: as shown in Remark~\ref{rem:expressivity_scope},
stratified graphical models~\citep{nyman2014stratified} require configuration matrices expressible only as a sum of Kronecker products, not as a single one, thus placing them outside the defined class.
The natural extension is to characterize the kernels of such mixture-like structures, and to determine when they still admit tractable generating sets.
The ultimate goal is a partition of all toric models by row space equivalence classes: viewing the row space as a span of $\{0,1\}$ vectors yields a $2^N$
decomposition that covers model space without overlap or omission.

Existing structure learning methods~\citep{drton2017structure} test one
candidate model at a time against the observed distribution, with no information about what lies outside the tested hypotheses.
In contrast, the $2^N$ row-space partition classifies all models simultaneously, making model selection a partition membership problem over an exhaustively defined hypothesis space.
MICs are the building blocks of this classification: their multiplicative and
additive compositions (Remarks~\ref{rem:mic-composition} and~\ref{rem:MIC-as-atom}) generate
the complete signature of every model in the Kronecker-stack class.
Because MICs are geometric invariants of the row space, they are determined
algebraically without parameter inference or optimization.
Therefore, any unknown structure in a discrete probability tensor is in
principle identifiable through MICs and their compositions, since $\{0,1\}$ vectors span the row space of the configuration matrix.
The row-space partition thus solves a true inverse problem, determining structure directly from algebraic constraints without prior hypothesis specification, in contrast to the forward approaches in existing structural learning.

However, exhaustive identification is computationally prohibitive: the number of candidate signatures grows exponentially in the number of variables and states.
A further combinatorial challenge arises in assembling identified MICs into larger structures.
Two directions are under investigation to make this tractable.
First, marginalization reduces the search to lower-dimensional projections,
identifies MICs there, and lifts them to the full space in a hierarchical search that prioritizes promising regions;
the use of bigram co-occurrence matrices in natural language processing may be viewed as a de facto instance of this strategy, since a bigram matrix is obtained by marginalizing a higher-order co-occurrence tensor over each pair of variables and concatenating the resulting matrices.
Second, the algebraic structure of toric ideals admits an interpretation in terms of group invariance~\citep{aoki2008largest}, and thus harmonic and group-theoretic analysis can exploit this structure to compress the signature search space and accelerate enumeration.

\section{Conclusion}
\label{sec:conclusions}

This paper proposed Algebraic Signature Analysis (ASA), a framework for
structural learning in probability tensors based on parameter-invariant
algebraic signatures. 
By treating the vanishing binomials of a toric model as its signature, we turned the ideal-variety correspondence into an operational procedure that identifies a model by signature matching, without parameter estimation. 
Restricting attention to the Kronecker-stack class of configuration
matrices made these signatures explicitly enumerable, and within this class we identified the minimum invariant constraint (MIC) as the atomic unit from which every signature is composed. 
This places independence, conditional independence, context-specific independence, and local rank-one structures within a single
class of invariance structures. 
Viewed through the row-space equivalence of configuration matrices, model selection becomes a partition membership problem over an exhaustively defined hypothesis space, in contrast to forward approaches that test one candidate structure at a time. 
The synthetic experiments showed that signature matching recovers model structures from finite samples and provides a computationally efficient alternative to likelihood-based model comparison. 
The language-data experiment showed that the same MIC-based structures are recoverable from corpus-derived tensors, where they correspond to
coherent semantic groupings obtained without supervision, opening a new
application avenue for algebraic statistics in computational linguistics. These results indicate that algebraic statistics can be extended from the descriptive study of predefined models toward structural learning from empirical distributions, and they point toward broader classes of invariant structures, higher-degree signatures, and further applications in computational linguistics.

\appendix
\numberwithin{equation}{section} 
\renewcommand{\theequation}{\thesection.\arabic{equation}}
\newtheorem{alemma}{Lemma}[section] 
\renewcommand{\thealemma}{\thesection.\arabic{alemma}}

\section{Supplemental proof to Theorem~\ref{thm:bijectivity}}
\label{app:proof2.5}
\begin{lemma}[Model equality and row-space equality]
\label{lem:model-rowspace-equivalence}
Let $U,V\subseteq \mathbb{R}^N$ be linear subspaces such that $\bm{1}\in U$ and $\bm{1}\in V$. 
Abusing notation, define $M(U):=\left\{\bm{p}\in \Delta^{N-1}_{>0}\mid\log \bm{p}\in U\right\}$,
where the logarithm is applied component-wise. Then
$M(U)=M(V)\iff U=V$.
Consequently, for homogeneous configuration matrices $A$ and $A'$, under the definition in Thorem~\ref{thm:bijectivity},
\begin{equation}
M(A)=M(A')
\;\iff \;
\operatorname{rowspan}(A)=\operatorname{rowspan}(A').
\end{equation}
\end{lemma}

\begin{proof}
The implication $U=V \Rightarrow M(U)=M(V)$ is immediate from the definition.
We prove the converse.

Assume that $M(U)=M(V)$. 
We first prove that $U\subseteq V$.
Let $\bm{q}\in U$ be arbitrary.
Define $Z(\bm{q}):=\sum_{i=1}^{N}\exp(q_i)$, 
and set $p_i:=\frac{\exp(q_i)}{Z(\bm{q})},\;i=1,\ldots,N$.
Then $\bm{p}\in \Delta^{N-1}_{>0}$. 
Moreover,
\begin{equation}
\log \bm{p}=\bm{q}-\log Z(\bm{q})\,\bm{1}.
\end{equation}
Since $\bm{q}\in U$ and $\bm{1}\in U$, we have $\log \bm{p}\in U$.
Hence $\bm{p}\in M(U)$. 
Intuitively, the homogeneity assumption ensures that any vector in the given space can be shifted along the all-one direction so as to become the logarithm of a normalized probability distribution.

By the assumption $M(U)=M(V)$, it follows that
$\bm{p}\in M(V)$, and therefore $\log \bm{p}\in V$.
Since $\bm{1}\in V$, we obtain
\begin{equation}
\bm{q}=\log \bm{p}+\log Z(\bm{q})\,\bm{1}\in V.
\end{equation}
Thus $U\subseteq V$. 
Reversing the roles of $U$ and $V$ gives $V\subseteq U$.
Therefore $U=V$.
Taking $U=\operatorname{rowspan}(A),\;V=\operatorname{rowspan}(A')$ proves the statement for homogeneous configuration matrices. 
In particular, by Definition~\ref{def:model}, the equality $M(A)=M(A')$ entails $\operatorname{rowspan}(A)=\operatorname{rowspan}(A')$.
\end{proof}

\begin{lemma}[The Graver basis generates the integer kernel]
\label{lem:graver-generates-lattice}
Let $A\in \mathbb{Z}^{d\times N}$ and let
\begin{equation}
L_A:=\ker_{\mathbb{Z}}A
=
\{\bm{z}\in\mathbb{Z}^N\mid A\bm{z}=0\}.
\end{equation}
Then the Graver basis $\graver(A)$ generates $L_A$ as an integer lattice:
\begin{equation}
\langle \graver(A)\rangle_{\mathbb{Z}}
=
L_A.
\end{equation}
\end{lemma}

\begin{proof}
Since $\graver(A)\subseteq L_A$ and $L_A$ is an integer lattice, we immediately have
\begin{equation}
\langle \graver(A)\rangle_{\mathbb{Z}}
\subseteq
L_A.
\end{equation}

For the reverse inclusion, let $\bm{z}\in L_A$ be nonzero. 
If $\bm{z}$ is conformally indecomposable in $L_A$, then $\bm{z}\in \graver(A)$ by definition. 
Hence $\bm{z}\in \langle \graver(A)\rangle_{\mathbb{Z}}$.

Suppose instead that $\bm{z}$ is conformally decomposable. 
Then there exist nonzero vectors $\bm{z}_1,\bm{z}_2\in L_A$ such that $\bm{z}=\bm{z}_1+\bm{z}_2$, 
and $\bm{z}_1\sqsubseteq \bm{z},\;\bm{z}_2\sqsubseteq \bm{z}$
Since the decomposition is nontrivial and conformal, both $\bm{z}_1$ and $\bm{z}_2$ have strictly smaller $\ell_1$-norm than $\bm{z}$.
If either $\bm{z}_1$ or $\bm{z}_2$ is conformally decomposable, we decompose it further. 
Repeating this procedure, the $\ell_1$-norm strictly decreases at each nontrivial decomposition step. 
Since the $\ell_1$-norm is a nonnegative integer, the process must terminate after finitely many steps.
At termination, $\bm{z}$ is expressed as a finite conformal sum of nonzero conformally indecomposable elements of $L_A$. 
By definition, these terminal elements belong to $\graver(A)$. 
Hence $\bm{z}\in \langle \graver(A)\rangle_{\Z}$ and thus, $L_A \subseteq \langle \graver(A)\rangle_{\Z}$, concluding $L_A=\langle \graver(A)\rangle_{\mathbb{Z}}$.
\end{proof}
See also \citet[Chapter 4]{sturmfels1996grobner}, \citet[Chapter 3]{onn2010nonlinear} and \citet[Chapter 4.6]{aoki2012markov}.
\begin{lemma}[Integer kernel spans the real kernel]
\label{lem:integer-kernel-spans-real}
Let $A \in \mathbb{Z}^{d \times N}$ be an integer matrix. Then
\begin{equation}
    \ker_{\mathbb{R}} A 
    = 
    \operatorname{span}_{\mathbb{R}}(\ker_{\mathbb{Z}} A).
\end{equation}
\end{lemma}

\begin{proof}
The inclusion $\operatorname{span}_{\mathbb{R}}(\ker_{\mathbb{Z}} A) \subseteq \ker_{\mathbb{R}} A$ 
is immediate: every $\bm{z} \in \ker_{\mathbb{Z}} A$ satisfies $A\bm{z} = 0$, 
and $\ker_{\mathbb{R}} A$ is closed under real linear combinations.

For the reverse inclusion, let $\{\bm{v}_1, \ldots, \bm{v}_k\}$ be a basis of $\ker_{\mathbb{R}} A$.
Since $A \in \mathbb{Z}^{d \times N}$, the system $A\bm{z} = \bm{0}$ has rational coefficients.
Gaussian elimination over $\mathbb{Q}$ produces a basis consisting entirely of rational vectors;
that is, we may assume $\bm{v}_j \in \mathbb{Q}^N$ for each $j = 1, \ldots, k$.
For each $j$, let $c_j \in \mathbb{Z}_{>0}$ be the least common multiple of the denominators 
of the entries of $\bm{v}_j$, and set $\bm{w}_j := c_j \bm{v}_j \in \mathbb{Z}^N$.
Then $A\bm{w}_j = c_j A\bm{v}_j = \bm{0}$, so $\bm{w}_j \in \ker_{\mathbb{Z}} A$.
Since $\bm{v}_j = c_j^{-1}\bm{w}_j \in \operatorname{span}_{\mathbb{R}}(\ker_{\mathbb{Z}} A)$
for each $j$, and $\{\bm{v}_1,\ldots,\bm{v}_k\}$ spans $\ker_{\mathbb{R}} A$, we conclude
$\ker_{\mathbb{R}} A \subseteq \operatorname{span}_{\mathbb{R}}(\ker_{\mathbb{Z}} A)$.
\end{proof}

\section{Monte Carlo calibration of the theoretically optimal threshold}
\label{app:mc_threshold_calibration}

This appendix describes the Monte Carlo procedure used to compute the theoretically calibrated threshold $\tau^*$ for the normalized binomial invariant. 
We consider a local $2 \times 2$ probability table associated with a single $2 \times 2$ minor. 
Let $H_0$ denote the local independence model and $H_1$ denote the local saturated model. 
The corresponding configuration matrices are given by
\begin{equation}
A_0
=
\begin{bmatrix}
1 & 1 & 0 & 0 \\
0 & 0 & 1 & 1 \\
1 & 0 & 1 & 0 \\
0 & 1 & 0 & 1
\end{bmatrix},
\qquad
A_1
=
\begin{bmatrix}
1 & 0 & 0 & 0 \\
0 & 1 & 0 & 0 \\
0 & 0 & 1 & 0 \\
0 & 0 & 0 & 1
\end{bmatrix}.
\end{equation}
For a parameter vector $\bm{\theta} \in [0,1]^4$, the probability vector generated by a configuration matrix $A$ is defined by
\begin{equation}
p_i(\bm{\theta};A)
:=
\frac{\exp(\bm{\theta}^{\top}\bm{a}_i)}
{\sum_{j=1}^{4}\exp(\bm{\theta}^{\top}\bm{a}_j)},
\end{equation}
where $\bm{a}_i$ denotes the $i$-th column of $A$. 

For each Monte Carlo replication $b=1,\ldots,B$, we draw
$
\bm{\theta}^{(b)}
\sim
\operatorname{Uniform}([0,1]^4),
$with $B=30000$. 
Using this parameter vector, we compute the two local probability vectors
$\bm{p}_{\mathrm{ind}}^{(b)}=
\bm{p}(\bm{\theta}^{(b)};A_{0})$ and
$\bm{p}_{\mathrm{sat}}^{(b)}=
\bm{p}(\bm{\theta}^{(b)};A_{1})$.

Although the full probability tensor has $K^3$ cells, each normalized binomial invariant is evaluated on only four cells. 
Therefore, the relevant sample size for a local $2 \times 2$ minor is approximated by the local effective sample size
$N_{\mathrm{eff}}=N\frac{4}{K^3}$.
This quantity represents the expected number of observations assigned to the four cells involved in a given local $2 \times 2$ minor, under a uniform allocation of the total sample size over the $K^3$ tensor cells.
For each pair $(K, N)$, where $K \in \{2,3,4,5\}$ and $N \in \{1000,10000,100000\}$, we simulate empirical probability vectors by multinomial sampling with the effective sample size $N_{\mathrm{eff}}$:
\begin{equation}
\hat{\bm{p}}_{\mathrm{ind}}^{(b)}
=
\frac{1}{N_{\mathrm{eff}}}
\operatorname{Multinomial}
\left(
N_{\mathrm{eff}},
\bm{p}_{\mathrm{ind}}^{(b)}
\right)
\end{equation}
and
\begin{equation}
\hat{\bm{p}}_{\mathrm{sat}}^{(b)}
=
\frac{1}{N_{\mathrm{eff}}}
\operatorname{Multinomial}
\left(
N_{\mathrm{eff}},
\bm{p}_{\mathrm{sat}}^{(b)}
\right).
\end{equation}
For each empirical probability vector, we evaluate the normalized binomial invariant
\begin{equation}
\tilde{m}(\hat{\bm{p}})
=
\frac{
\hat{p}_{11}\hat{p}_{22}
-
\hat{p}_{12}\hat{p}_{21}
}{
\hat{p}_{11}\hat{p}_{22}
+
\hat{p}_{12}\hat{p}_{21}
}.
\end{equation}
This gives two Monte Carlo samples:
$\left\{
\left|
\tilde{m}(\hat{\bm{p}}_{\mathrm{ind}}^{(b)})
\right|
\right\}_{b=1}^{B}$ as the finite-sample distribution of the invariant under $H_0$, 
$\left\{
\left|
\tilde{m}(\hat{\bm{p}}_{\mathrm{sat}}^{(b)})
\right|
\right\}_{b=1}^{B}$ as the corresponding distribution under $H_1$.
For a candidate threshold $\tau$, the Type I error rate is estimated by
\begin{equation}
\hat{\alpha}(\tau)
=
\frac{1}{B}
\sum_{b=1}^{B}
I
\left(
\left|
\tilde{m}(\hat{\bm{p}}_{\mathrm{ind}}^{(b)})
\right|
>
\tau
\right),
\end{equation}
where $I(\cdot)$ denotes the indicator function. 
Similarly, the Type II error rate is estimated by
\begin{equation}
\hat{\beta}(\tau)
=
\frac{1}{B}
\sum_{b=1}^{B}
I
\left(
\left|
\tilde{m}(\hat{\bm{p}}_{\mathrm{sat}}^{(b)})
\right|
\leq
\tau
\right).
\end{equation}

Assuming symmetric prior probabilities over the two hypotheses, $\pi_0=\pi_1=1/2$, the Monte Carlo estimate of the total risk is
\begin{equation}
\hat{R}(\tau)
=
\pi_0\hat{\alpha}(\tau)
+
\pi_1\hat{\beta}(\tau).
\end{equation}
The theoretically calibrated threshold is then defined as
\begin{equation}
\tau^*
=
\operatorname*{arg\,min}_{\tau \in \mathcal{T}}
\hat{R}(\tau),
\end{equation}
where $\mathcal{T}$ is the grid of candidate threshold values. 
This procedure is repeated independently for each pair $(K,N)$, yielding the calibrated thresholds reported in Table~\ref{tab:tau_theory}.

\section{Statistical test on normalized binomial invariant \label{sec-stat-normalized}}
In Section \ref{sec-stats}, we consider a probability vector $\bm{p} \in \Delta_{N-1}$
and some threshold for a given normalized binomial invariant Eq.~\eqref{eq-normalized-binomial}. Here we consider its more general form: 
\begin{equation}
\nonumber
    \tilde{m}_{i,j}(\bm{p}) := \frac{\prod_{k=1}^{m}p_{i_k} - \prod_{k=1}^{m}p_{j_k}}{\prod_{k=1}^{m}p_{i_k} + \prod_{k=1}^{m}p_{j_k}},
\end{equation}
where $i = (i_1, \ldots, i_k) \in \overline{N}^{nk}$ and $j = (j_1, \ldots, j_k) \in \overline{N}^{nk}$ for some set of integers, and the $i$ and $j$ has distinct integers, meaning
$i_{m} \neq i_{m'}$ and $j_{m} \neq j_{m'}$ for $m \neq m'$
and 
$i_{m} \neq j_{m'}$ for any $1 \le m, m' \le k$. 

In this appendix, we give a theoretical justification why the normalized binomial invariants Eq.~\eqref{eq-normalized-binomial} were used for model selection, rather than the other statistics such as unnormalized binomial invariants $\prod_{k=1}^{m}p_{i_k} - \prod_{k=1}^{m}p_{j_k}$,
by showing a statistical hypothesis test on the normalized binomial invariant
(Corollary \ref{cor-mij}).

The following lemma implies
the normalized binomial invariants Eq.~\eqref{eq-normalized-binomial} is a natural choice to test whether it is significantly different from zero, 
under a null hypothesis that the probability vector $\bm{p} \in \Delta_{N-1}$
follows Dirichlet distribution 
probability distribution of $D(\bm{p} \mid \alpha)$ with the parameter $\alpha \in \mathbb{R}_+^N$.

\begin{lemma}
\label{lem-pij}
    Suppose $\bm{p} \in \Delta_{N-1}$ is drawn from a Dirichlet distribution 
    with the parameter $\alpha \in \mathbb{R}_+^N$:
\begin{align}
\label{eq-Dirichlet-p}
    D(p \mid \alpha) 
    :=
    \frac{\Gamma(\alpha_0)}{\prod_{i=1}^{N}\Gamma(\alpha_i)}
    \prod_{i=1}^{N}p_i^{\alpha_{i}},
\end{align}
where $\alpha_0 := \sum_{i=1}^{N}\alpha_{i}$.
Then, 
the normalized binomial $\hat{p}_{i,j} := \frac{ p_{i} - p_{j} }{ p_{i} + p_{j} }$ 
for $i \neq j$
follows the affine-transformed Dirichlet (Beta) distribution: 
\begin{align}
    \Pr\left( \hat{p}_{i,j} \right) 
    = 
    D\left( \left( \frac{ 1 + \hat{p}_{i,j} }{2}, \frac{1 - \hat{p}_{i,j}}{2} \right) \mid \alpha_{ij} \right).
\end{align}
\end{lemma}

\begin{proof}
The Dirichlet distribution 
\eqref{eq-Dirichlet-p} of the probability vector $\bm{p}$
is equivalent to consider 
independently gamma-distributed random variables $q_1, \ldots, q_N$, with the $\Pr\left( q_i \mid \alpha_i, 1\right) = q_i^{\alpha_i-1}e^{- q_i}$ for $i = 1,\ldots,N$, 
with which $p_i$ is redefined by $p_{i} := \frac{ q_i }{\sum_{k=1}^{N} q_i}$. 
With this transformation, we have 
$\hat{p}_{i,j} = \frac{ q_{i} - q_{j} }{ q_{i} + q_{j} }$. 
The joint probability density function of $q$ is given by 
\begin{align}
    \Pr( q \mid \alpha ) = 
    \frac{ \prod_{k \in \overline{N}}q_k^{\alpha_{k}-1}}{
    \prod_{k \in \overline{N}}\Gamma(\alpha_i) }
    e^{ -\sum_{k \in \overline{N}}q_{k} }
    =
    \frac{ q_i^{\alpha_i-1}q_j^{\alpha_j-1} }{
    \prod_{k \in \overline{N}}\Gamma(\alpha_i) }
    \prod_{k \in \overline{N} \setminus \{i, j\}}q_k^{\alpha_k-1}
    e^{ -q_i -q_j - \sum_{k \in \overline{N} \setminus \{i,j\}}q_{k} }.
\end{align}
In the above, we replace the random variable $q_j$ with $\hat{p}_{i,j}$ by 
$q_{j} = \frac{ 1-\hat{p}_{i,j} }{ 1+\hat{p}_{i,j} }q_{i}$
and $\frac{\mathrm{d}(-q_{j})}{\mathrm{d}\hat{p}_{i,j}} = \frac{ 2 q_{i} }{ (1 + \hat{p}_{i,j})^2 }$.
By this variable transformation, 
the marginalized probability density function of $\hat{p}_{ij}$ is 
\begin{align}
\Pr\left( \hat{p}_{i,j} \mid \alpha \right)
&=
\int_{
}
\frac{ q_i^{\alpha_i-1} 
\left(
\frac{ 1-\hat{p}_{i,j} }{ 1 + \hat{p}_{i,j} }q_{i}
\right)^{\alpha_j-1}
}{
    \prod_{k \in \overline{N}}\Gamma(\alpha_i) 
}
\prod_{k \in \overline{N} \setminus \{i, j\}}q_k^{\alpha_k-1}
    e^{ -q_i -\frac{ 1-\hat{p}_{i,j} }{ 1+\hat{p}_{i,j} }q_{i} - \sum_{k \in \overline{N} \setminus \{i,j\}}q_{k} }
\frac{ 2q_{i} }{ ( 1 + \hat{p}_{ij} )^2 }
\prod_{k \in \overline{N}\setminus \{j\}}\mathrm{d}q_{k},
\\
&=
\int_{}
\frac{ q_i^{\alpha_i+\alpha_j-1} }{
    \prod_{k \in \overline{N}}\Gamma(\alpha_i) }
\frac{ (1-\hat{p}_{i,j})^{\alpha_j-1} }{ (1+\hat{p}_{i,j})^{\alpha_j+1} }
\prod_{k \in \overline{N} \setminus \{i, j\}}q_k^{\alpha_k-1}
    e^{ -\left( \frac{ 2 }{ 1+\hat{q}_{i,j} } \right) q_{i} - \sum_{k \in \overline{N} \setminus \{i,j\}}q_{k} }
2
\prod_{k \in \overline{N}\setminus \{j\}}\mathrm{d}q_{k},
\end{align}
By the definition of gamma function,
$\frac{ \Gamma(\alpha_i) }{ \theta^{\alpha_i} }= \int_{0}^{\infty} q^{\alpha_{i}-1}e^{-\theta q} \mathrm{d}q$. Applying this to the above, we have
\begin{align}
\Pr\left( \hat{p}_{i,j} \mid \alpha \right)
&=
\frac{\Gamma(\alpha_i + \alpha_j)}{\Gamma(\alpha_i) \Gamma(\alpha_j)}
\frac{ (1-\hat{p}_{i,j})^{\alpha_j-1} }{ (1+\hat{p}_{i,j})^{\alpha_j+1} }
\frac{2 \left( 1 + \hat{p}_{i,j} \right)^{\alpha_i + \alpha_j} }{
2^{\alpha_i + \alpha_j}
}
\\
&=
\frac{\Gamma(\alpha_i + \alpha_j)}{\Gamma(\alpha_i) \Gamma(\alpha_j) 2^{\alpha_i + \alpha_j - 1}}
(1-\hat{q}_{i,j})^{\alpha_j-1}
(\hat{q}_{i,j} + 1)^{\alpha_i -1}
.
\end{align}
This probability density function $\Pr\left( \hat{p}_{i,j} \mid \alpha \right)$ implies 
$\hat{p}_{i,j}$ is distributed 
by the Dirichlet (Beta) distribution 
$D\left( \left( \frac{1+\hat{p}_{i,j}}{2}, \frac{1-\hat{p}_{i,j}}{2} \right) \mid (\alpha_i, \alpha_j) \right)$
with 
$\hat{p}_{ij} = 2b - 1$
with a random variable $b$ 
distributed by 
$D( (b, 1-b) \mid ( \alpha_{i}, \alpha_{j} ) )$.
\end{proof}

In the statistical hypothesis test, it is natural to have no bias on $\bm{p}$. Namely, the expectation of the probabilities 
hold 
$\int_{} p_{i}D(\bm{p} \mid \alpha) = 
\frac{\alpha_{i}}{\alpha_{0}}
=
\frac{\alpha_{j}}{\alpha_{0}}
=
\int_{} p_{j}D(\bm{p} \mid \alpha)$, 
which implies 
$\alpha_{i} = \alpha_{j}$,
for any $i, j \in \overline{N}$.
Under this unbiased assumption, 
there is $\beta > 0$ such that
$\alpha = \beta (1, \ldots, 1) \in \mathbb{R}_{+}^{N}$.


\begin{corollary}
\label{cor-mij}
Under the unbiased assumption that 
$\bm{p}$ is drawn from some unbiased Dirichlet distribution $D( \bm{p} \mid \beta (1, \ldots, 1) )$, 
the probability for 
the normalized binomial invariant \eqref{eq-normalized-binomial} to 
satisfy 
$-\tau \le \tilde{m}_{i,j}(\bm{p})  \le \tau$ for some threshold $0 \le \tau_{0}, \tau_{1} \le 1$ is  
given by the following incomplete Beta function
\begin{align}
\label{eq-p-value-mij}
    \Pr\left( -\tau_0 \le \tilde{m}_{i,j}(\bm{p})  \le \tau_{1} \mid \beta \right)
    =
    \frac{1}{B(\beta,\beta)}\int_{\frac{1-\tau_{0}}{2}}^{\frac{1+\tau_{1}}{2}}t^{\beta -1} (1-t)^{\beta -1}
    \mathrm{d}t.
\end{align}
\end{corollary}

\begin{proof}
For $m=1$, Lemma~\ref{lem-pij} with $\alpha = \beta(1,\ldots,1)$ implies that the normalized binomial invariant Eq.~\eqref{eq-normalized-binomial} $\tilde{m}_{i,j}(\bm{p})$
is distributed 
by the Dirichlet (Beta) distribution 
\begin{align}
\label{eq-mij}
    \Pr\left( \tilde{m}_{i,j}(\bm{p}) \right) 
    = 
    D\left( \left( \frac{ 1 + \tilde{m}_{i,j}(\bm{p}) }{2}, \frac{1 - \tilde{m}_{i,j}(\bm{p})}{2} \right) \mid (\beta, \beta) \right).
\end{align}
Then taking partial integral 
over the interval
$-\tau_{0} \le \tilde{m}_{i,j}(\bm{p})  \le \tau_{1}$, we have \eqref{eq-p-value-mij}.
\end{proof}

Corollary~\ref{cor-mij} justifies use of the normalized binomial invariant $\tilde{m}_{i,j}(\bm{p})$, when we wish to statistically test whether the (unnormalized) binomial invariant
$\prod_{k=1}^{m}p_{i_k} - \prod_{k=1}^{m}p_{j_k}$
is sufficiently close to zero.
Note that 
the probability 
$\Pr\left( -\tau_{0} \le \tilde{m}_{i,j}(\bm{p})  \le \tau_{1} \mid \beta \right)$ in 
Corollary~\ref{cor-mij}
is essentially independent from 
a particular choice of $i,j$ -- no index $i,j$ at its right hand side of \eqref{eq-p-value-mij}.
This statistical property of the normalized binomial invariant $\tilde{m}_{i,j}(\bm{p})$
is quite useful, when we test whether $\tilde{m}_{i,j}(\bm{p})$ is near zero, as the p-value of such null hypothesis depends only on the choice of some threshold $(\tau_{0},\tau_{1})$. 
In other words, an appropriate threshold $(\tau_{0}, \tau_{1})$ to be associated to a fixed p-value for 
the other form, such as unnormalized binomial invariants $\prod_{k=1}^{m}p_{i_k} - \prod_{k=1}^{m}p_{j_k}$,
would be affected by 
potential variance of the probability vector $p_{k}$.

\section{Algorithm for ASA-D on the corpus}
\label{app:asad_corpus}

This appendix documents the procedure of ASA-D (Algebraic Signature Analysis for Discovery) that produced the rank-one SVO sub-tensors reported in Section~\ref{sec:linguistic}. 
The pipeline takes a corpus as input and returns a catalogue of $1{,}889$ near rank-one sub-tensors.
We describe it in five stages: data extraction and tensor
construction, pre-processing, MIC detection, sub-tensor assembly, and filtering.
All parameter values stated below are the confirmed settings used for the reported results.
\subsection{Data extraction and tensor construction}
\label{app:data}

We extracted subject--verb--object triples from the Dependency-parsed Common Crawl corpus \citep[DepCC;][]{panchenko-etal-2018-building}, 
the largest dependency-annotated English corpus available, 
comprising approximately $365$ million documents and $14.3$ billion sentences. 
Co-occurrence counts of lemmas annotated as \texttt{nsubj}, \texttt{root}, and \texttt{dobj} were accumulated into an integer tensor $T \in \mathbb{Z}_{\geq 0}^{500 \times 500 \times 500}$, retaining the top $500$ lemmas for each axis by frequency.
We denote each axis of the tensor $S$ for subject, $V$ for verb and $O$ for object, and their elements $s\in S, v\in V, o\in O$, respectively.
The tensor contains $2.3$ billion triplets with $64.7$ million nonzero cells.
The index-to-lemma maps for the three axes are stored alongside the tensor.
$S\times V$ slices fixing an object word $o$, $M_o := T[\cdot,\cdot,o] \in \mathbb{Z}_{\geq 0}^{500\times 500}$, are the unit of all subsequent processing, and each slice is treated independently.

\subsection{Pre-processing}
\label{app:preproc}
 
Raw co-occurrence counts span several orders of magnitude and follow the Zipf law (see Section~\ref{subsec:setup}). 
We work throughout in the log domain, adding one to each raw count to keep zero counts finite,
$ L_o := \log(1 + M_o)$ for $M_o = T[\cdot,\cdot,o] \in \mathbb{Z}_{\geq 0}^{500\times 500}$.
This transformation is applied to every slice before any minor is evaluated.

\subsection{MIC detection}
\label{app:mic}

For each $S\times V$ slice we identify the pairs of cells whose normalised $2\times 2$ minor vanishes. 
Detection proceeds on the cells carrying positive association, the others being treated as structural zeros.
 
We define structural zeros through the sign of the pointwise mutual information (PMI).
Writing $r_s = \sum_v M_o[s,v]$, $c_v = \sum_s M_o[s,v]$, and $N_o = \sum_{s,v} M_o[s,v]$, the independence model predicts $E_o[s,v] := r_s c_v / N_o$ for all $s\in S$ and $v \in V$. 
In contrast, a cell $(s,v,o)$ with  $M_o[s,v] \leq E_o[s,v]$, i.e., that has a negative PMI value, is defined as a structural zero here that reflects linguistically inadmissible combination of words.
Thus, we are only interested in the complementary cells that have a positive PMI (PPMI).
We additionally require a minimum count to suppress low-count noise on top of the PPMI condition forming the \emph{active support}:
\begin{equation}
  \label{eq:active}
  M_o[s,v] > E_o[s,v]
  \quad\text{and}\quad
  M_o[s,v] \geq \tau,
  \qquad \tau = 3 .
\end{equation}
 
On the active support we evaluate minors of the log-transformed slice $L_o$.
For two subjects $s_1, s_2$ sharing at least two active verbs, and 
for any two such verbs $v_1, v_2$, 
we form the symmetric normalised minor
\begin{equation}
  \label{eq:normalised_minor}
  m(s_1,s_2,v_1,v_2) \;=\;
  \frac{L_o[s_1,v_1]\,L_o[s_2,v_2] - L_o[s_1,v_2]\,L_o[s_2,v_1]}
       {L_o[s_1,v_1]\,L_o[s_2,v_2] + L_o[s_1,v_2]\,L_o[s_2,v_1]} ,
\end{equation}
which lies in $(-1,1)$ and vanishes exactly when the four log-counts are
proportional, that is when the corresponding $2\times 2$ block of $L_o$ has
rank one. 
A quadruple $(s_1,s_2,v_1,v_2)$ is recorded as a \emph{MIC hit} for object
$o$ when $|m| \leq \varepsilon$ with $\varepsilon = 0.05$. 
The output is the set of all such hits across all objects, indexed by $o$ for slice-wise retrieval in the next stage.

\subsection{Sub-tensor assembly}
\label{app:assembly}

Within each object $o$, the MIC hits define a structure on subjects and verbs:
each subject pair $\{s_1,s_2\}$ carries the set of verb pairs with which it forms a vanishing minor. 
From this structure we assemble candidate rank-one sub-tensors in four sub-steps.

\subsubsection{Seed pairs.}
We choose subject pairs as seeds by two complementary criteria. 
\begin{enumerate}
    \item The \emph{principal} seeds are the subject pairs of highest normalised overlap $|\,V_{s_1}\cap V_{s_2}\,| \,/\, |\,V_{s_1}\cup V_{s_2}\,|$, where $V_s$ is the set of verbs in which $s$ participates in a hit, subject to sharing at least ten verbs; we keep the top twenty. 
    \item The \emph{rare} seeds are the subject pairs that share at least three verbs but have the smallest global frequency sum, above a floor that excludes pure noise; we keep the twenty rarest. The rare seeds target low-frequency structure that frequency-driven methods overlook.
\end{enumerate}

\subsubsection{Verb core.}
For each seed pair $(s_i,s_j)$ to the fixed $o$, the shared verb pairs $(v_k, v_l) $ where the normalized binomial $m(s_i,s_j,v_k,v_l) $ vanishes, are read from the list of MIC hits.
We construct a graph with these verb pairs as edges and their constituent verbs as vertices, and enumerate the maximal cliques of size at least three. 
Each such clique is a set of verbs mutually exchangeable with respect to the seed subjects; 
we call it a \emph{verb core} $V_{core}$, which is by definition a complete graph on its vertices.

\subsubsection{Subject expansion.}
Given a verb core $V_{core}$, we admit further subject words $s_p$ for $p\neq i,j$ beyond seed subjects when $s_p$ together with either of $s_i$ or $s_j$ makes MIC with many verbs in $V_{core}$ above a certain density.
For a candidate subject $s_p$, an edge $\{v_k,v_l\}$ is \emph{covered} when $s_p$ paired with $s_i$ or with $s_j$ forms a vanishing minor on that verb pair, that is when $\{v_k,v_l\}$ appears among the recorded MIC hits of $(s_p,s_i)$ or $(s_p,s_j)$. 
We seek the subjects whose covered edges fill the verb core densely:
the fraction of covered verb-pair edges must be at least $0.80$ and 
the fraction of covered verb vertices at least $0.95$, making sure that adding $s_p$ will not destroy the preexisting clique.
At most fifteen subjects are added per core to keep a set of words manually traceable. 
For the principal seeds, the resulting blocks are then merged when their verb sets have Jaccard similarity at least $0.80$, which consolidates near-duplicate cores arising from different but overlapping seeds; rare-seed blocks are kept separate with intention to detect a very local structure.
This extended pair of sets $(S', V')$ where $S', V'$ are subsets of the full subject and verb vocabularies is called a $\emph{subject-verb block}$

\subsubsection{Object extension.}
Each subject-verb block $(S',V')$ is now extended along the object axis by collecting the objects whose slice over $(S',V')$ is proportional, cell by cell, to that of the seed object. 
In other words, the cell-wise ratio between the two vectorized $S'\times V'$ slices at the seed object and a candidate object is constant.
We measure the deviation of the candidate slice from exact proportionality to the seed slice in the log domain, defined by
\begin{equation}  \label{eq:norm_lr_std}
  \overline{\sigma}_o := \operatorname{std}
  \bigl\{\, \log \tfrac{L[s,v,o]}{L[s,v,o_{\mathrm{seed}}]}
  \;:\; (s,v)\in S'\times V' \,\bigr\}.
\end{equation}
We admit a candidate object $o$ if $\overline{\sigma}_o \leq \sigma_{\max}$.
If two slices are exactly parallel, the cell-wise log-ratios are identical
and $\overline{\sigma}_o = 0$.
To ensure that this dispersion is estimated from reliable counts rather than from sampling noise, we impose the additional condition that the co-occurrence mass of $(S',V')$ at the candidate object is at least a minimum count.
For principal-seed blocks we use $\sigma_{\max}=0.35$ and minimum count $5$; for rare-seed blocks $\sigma_{\max}=0.40$ and minimum count $3$.
We apply the procedure with two parameter settings because the two seed types address frames of different frequency.

The result is a three-way block $(S',V',O')$ over the selected objects.
Running this assembly with every object in turn as the seed object, over all $500$ objects, produces $12{,}782$ candidate sub-tensors.
 
\subsection{Filtering and selection}
\label{app:filtering}

The assembly produces $12{,}782$ candidate sub-tensors. 
The same structure can be recovered from several seed objects, so the output list contains duplicates, and the assembly criteria (MIC hits and the dispersion of
\eqref{eq:norm_lr_std}) do not directly test rank-one quality on the final
block. 
We therefore score each candidate by the metrics of Section~\ref{app:metric-definition}, the weighted rank-one fit $r_1$, the lift, and the NMF recall, and retain those that pass fixed thresholds.

A candidate is retained when $r_1 \geq 0.97$, its mean lift is at least $50$,
its dispersion satisfies $0.001 < \overline{\sigma}_\rho \leq 0.25$, and it
spans at least three subjects and three objects. 
The reason we impose a lower bound on $\overline{\sigma}_\rho$ is to exclude
trivial single-object cores: 
when no object beyond the seed is admitted, the dispersion is taken over the seed alone and is identically zero.
We further reject any block whose subject set contains a function word, 
since such words enter many frames for purely syntactic reasons.
These criteria yield $1{,}889$ sub-tensors drawn from $453$ distinct object words.

For the discussion in Section~\ref{sec:linguistic} we rank the retained sub-tensors by a composite score that increases with lift, rank-one fit, dispersion tightness, and object extent, and select five examples from the top that span distinct semantic domains. 

\subsection{Permutation experiment}
\label{app:permutation}

The null hypothesis is that the recovery is explained by the per-slice
distribution of values alone, independent of how those values are arranged
across $(s,v)$ pairs.

For each object index $o \in \{0, \dots, 499\}$ we take the $S\times V$ slice
$T[:,:,o] \in \mathbb{Z}_{\geq 0}^{500 \times 500}$ and apply a uniformly random
permutation to all $250{,}000$ cells within that slice.
Thus, only the per-slice total $\sum_{s,v} T[s,v,o]$ is preserved.
The row sums, the column sums, and the location of the nonzero entries are all
destroyed.
We apply to the permuted tensor the same pipeline used for the real data.

Table~\ref{tab:permutation} reports the three stages of the pipeline on the real
data and on the permutation null.

\begin{table}[htbp]
\centering
\caption{Pipeline output on the real tensor and on the permutation null
(seed $0$). Destroying the joint association removes $99.0\%$ of the cores at
extraction and all of the candidates after filtering.}
\label{tab:permutation}
\begin{tabular}{lrr}
\toprule
Stage & Real data & Permutation null \\
\midrule
MIC hits           & $410{,}117{,}056$ & $32{,}706{,}603$ \\
Cores (extraction)      & $12{,}782$            & $128$ \\
Final candidates        & $1{,}889$             & $0$ \\
\bottomrule
\end{tabular}
\end{table}

The number of MIC hits does not collapse under permutation, because a vanishing
$2 \times 2$ minor can occur by chance among small shuffled values.
The collapse appears at the extraction stage, where the core count falls from
$12{,}782$ to $128$, and is complete after filtering, where the candidate count
falls from $1{,}889$ to $0$.
The $128$ residual cores are all confined to a single object slice
($O\text{-size} = 1$), all have zero spread in the normalized log-ratio along the
object axis.

\section{Metric definitions}
\label{app:metric-definition}

We provide definitions and supplemental comments on the three metrics
used in Section~\ref{sec:linguistic}.

\vspace{0.3cm}
\textbf{Rank-1 fit} quantifies how well the extracted sub-tensor
$T[S', V', O']$ is approximated by a rank-1 tensor
$\hat{T} = \mathbf{u}\otimes\mathbf{v}\otimes\mathbf{w}$,
defined as:
\[
\mathrm{Rank\text{-}1\ fit} = 1 -
\frac{\|T[S',V',O'] - \hat{T}\|_F^2}{\|T[S',V',O']\|_F^2},
\]
where $\hat{T}$ is obtained by alternating optimization
over the three factor vectors $\mathbf{u}\in\mathbb{R}^{|S'|}$,
$\mathbf{v}\in\mathbb{R}^{|V'|}$, $\mathbf{w}\in\mathbb{R}^{|O'|}$,
first initialized via the leading singular vectors of the
mode-3 unfolding of $T[S',V',O']$, and then refined by
alternately updating each factor vector with the other
two fixed until convergence.
A value of $100\%$ indicates that the MIC condition holds
exactly across all cells in the sub-tensor.

\vspace{0.3cm}
\textbf{Lift} measures the degree to which the object words in a sub-tensor
are \emph{selectively} associated with the frame,
defined as the mean ratio of the within-frame object frequency
to its global frequency across the full corpus:
\[
  \mathrm{Lift} \;=\;
  \frac{1}{|O'|}\sum_{o \in O'}
  \frac{P(o \mid s \in S',\, v \in V')}{P(o)}.
\]
A lift of $1$ indicates no enrichment; values well above $1$
signal that the object words co-occur with the extracted
subject--verb pairs far more often than chance.

\vspace{0.3cm}
\textbf{NMF recall} measures the ability of standard matrix factorization
to recover the same local structures as ASA-D.
We apply non-negative matrix factorization with $K = 100$ components
to the full $500 \times 500$ slice matrix at each object value $o \in O'$.
For each extracted sub-tensor, NMF recall is defined as the
mean over $O'$ of the best-component recall:
\begin{equation}
  \mathrm{NMFrecall} \;=\;
  \frac{1}{|O'|}\sum_{o \in O'}
  \max_{k}\,
  \mathrm{Recall}_{S,k}(o)\times \mathrm{Recall}_{V,k}(o),
\end{equation}
where $\mathrm{Recall}_{S,k}(o)$ is the fraction of $S'$ appearing
in the top-$|S'|$ subjects of NMF component $k$ at slice $o$,
and analogously for $V$.
A value of $100\%$ would mean that some single NMF component
simultaneously recovers all subjects and verbs in the extracted sub-tensor.

\section*{Acknowledgments}
This work was supported by JSPS KAKENHI Grant Numbers JP22K11932, JP24KJ1202, JP26H02527, and JST CREST JPMJCR23P4.

\bibliographystyle{plainnat}
\bibliography{ref}

\end{document}